\documentclass[journal]{IEEEtran}
\usepackage{mathrsfs}

\usepackage{amsmath, amsthm, amssymb}
\usepackage{mathtools}
\usepackage{cite}
\usepackage{graphicx}
\usepackage{caption}
\usepackage{subcaption}
\usepackage{xcolor}
\usepackage{bbm}
\usepackage{algorithm}
\usepackage{algorithmicx}
\usepackage[noend]{algpseudocode}
\usepackage[margin=2cm]{geometry}
\usepackage{rotating}
\usepackage{pdflscape}
\usepackage{CJK}

{
      \theoremstyle{plain}
      \newtheorem*{assumption1}{Assumption 1}

      \newtheorem{proposition}{Proposition}
      \newtheorem{theorem}{Theorem}

}
\algrenewcommand{\algorithmicrequire}{\textbf{Input:}}
\algrenewcommand{\algorithmicensure}{\textbf{Output:}}
\algnewcommand{\And}{\State\textbf{and}}

\begin{document}
\begin{CJK}{UTF8}{gbsn}

\title{Stability Constrained Mobile Manipulation Planning on Rough Terrain}
\author{Jiazhi Song(宋佳智)$^{1}$ and Inna Sharf$^{2}$
\thanks{Both authors are with department of Mechanical Engineering, McGill University, Montreal, QC H3A 0C3, Canada $^{1}${\tt\small jiazhi.song@mail.mcgill.ca} $^{2}${\tt\small inna.sharf@mcgill.ca}
}%
}

\maketitle

\begin{abstract}
This paper presents a framework that allows online dynamic-stability-constrained optimal trajectory planning of a mobile manipulator robot working on rough terrain. First, the kinematics model of a mobile manipulator robot, and the Zero Moment Point (ZMP) stability measure are presented as theoretical background. Then, a sampling-based quasi-static planning algorithm modified for stability guarantee and traction optimization in continuous dynamic motion is presented along with a mathematical proof. The robot's quasi-static path is then used as an initial guess to warm-start a nonlinear optimal control solver which may otherwise have difficulties finding a solution to the stability-constrained formulation efficiently. The performance and computational efficiency of the framework are demonstrated through an application to a simulated timber harvesting mobile manipulator machine working on varying terrain. The results demonstrate feasibility of online trajectory planning on varying terrain while satisfying the dynamic stability constraint.
\end{abstract}

\section{Introduction}
\IEEEPARstart{M}{obile} manipulation is a popular topic in robotics research due to the omnipresence of this task in robotic applications. From indoor assistance \cite{huang2000} to Mars exploration \cite{howard2007}, mobile manipulators find application in various industries. The versatility of mobile manipulator robots is due to their simple yet effective construction: a certain number of manipulator arms mounted on top of a mobile platform that is either wheeled or tracked. In order to expand the application of mobile manipulators to uncontrolled outdoor environments, research efforts have also been spent on increasing the resiliency of mobile manipulators to their operating environment. Among different types of disturbances that can arise in outdoor environments, terrain variation \cite{papa1996} is among the most common, and it can have a large negative impact on the performance of a mobile robot if rollover or sliding occurs. 

Heavy equipment commonly used in mining, logging, and construction such as excavators, feller bunchers, and loaders can also be treated as mobile manipulators. Due to the slow development of automation in these industries, the needs related to terrain variation have rarely been addressed by the mobile manipulation research community. The aforementioned machines are prone to roll over as they have a high center of mass, manipulate heavy objects, and inevitably have to operate on slopes and in adverse weather conditions. The prevention of rollover and sliding is also of great importance to the machines employed in these industries, as this type of failure poses numerous risks to the operator, machine itself, and the environment. This paper aims to introduce an approach to allow robots to work in terrain-varying environments while maintaining or increasing their efficiency to also benefit the safety, productivity, and energy efficiency of future autonomous industrial applications. 

In order to ensure the upright stability of robots during mobile manipulation, passive methods such as novel mechanical design \cite{lindroos2017,chen2018} and steep slope avoidance \cite{pai1998,ge2002} can be utilized. However, complicated mechanical designs may not be practical in industrial applications, and slope avoidance misses out on the reconfiguration capability of mobile manipulators and may lead to a drastic reduction of their operational range. Therefore, in this article, we take a more proactive approach to plan the motion and configuration of a mobile manipulator robot to best allow stability constraint satisfaction so that a desired task is more likely to be completed. Before the proposed approach is introduced, past literature on stability measures, trajectory planning, stability constrained trajectory planning, and autonomous industrial machines is reviewed.

\subsection{Stability Measures}
Several measures designed to quantify the proximity to a rollover condition through force and moment measurement include the static force and moment analysis introduced in \cite{gibson1971}, the force-angle stability measure in \cite{papa1996,diaz2005,mosa2011}, and the lateral load transfer (LLT) parameter employed in \cite{bouton2007,bouton2010,denis2016}. However, these measures are based on the forces/moments experienced at joints or wheels and cannot provide direct guidelines for trajectory planning. Considering the complex structure of humanoid robots and the effect of momentum caused by different motions, the zero moment point (ZMP) stability measure was introduced in \cite{vuko1972} to achieve quantitative stability inference based on joint motions and robot inertia parameters. 

Since the mobile manipulation planning has to account for the robot's and the manipulated object's inertia parameters and the robot's joint motions, the ZMP stability measure naturally becomes the criterion of choice in this paper. The ZMP measure may be considered as too conservative in a walking robot scenario as it focuses on the overall moment along a supporting edge. In other words, ZMP instability corresponds to the beginning of rolling but does not necessarily mean the robot will rollover. However, the criterion is appropriate for a mobile manipulator robot with tracked/wheeled base since the base is intended to maintain full contact with the ground during ordinary operation. 

\subsection{Mobile Manipulation Planning}
There is a variety of results on manipulator and mobile manipulator optimal trajectory planning \cite{shin1985, bobrow1985,seraji1998, versch2016} but none of the previous works address the problem of rollover and sliding prevention. 

With an emphasis on arm-base coordination, theoretical results for coordinated end-effector trajectory tracking with holonomic and nonholonomic base using a continuous-time kinematic model of the robot were provided in \cite{seraji1998}. Building on the results in \cite{seraji1998}, a 5-DoF arm mounted on a nonholonomic mobile base to draw lines on walls is modeled as a redundant system in \cite{padois2006}. Redundancy of the system was utilized to maximize manipulability, reduce impact force from contact with wall, and avoid obstacles. A coordinated motion generation method for a mobile manipulator grasping an object with uncertainty is presented in \cite{chen2015}. The mobile base is controlled to maximize arm reachability by positioning the arm so that the object is located at a ``sweet" spot within the arm's workspace. This method only considers the kinematic model of a robot and would require another layer of control in order to track the prescribed trajectory.

Using optimal control formulation, in \cite{avanzini2015}, a mobile manipulator with a 3-DoF holonomic base and a 5-DoF arm was able to achieve trajectory tracking under various constraints while avoiding obstacles, by employing model predictive control. In order to obtain linear constraints for the online optimization, object avoidance constraints were approximated by assuming the robot is moving at maximum velocity. However, this technique would drastically reduce the mobile capability of a robot in a cluttered environment, such as a forest.

\subsection{ZMP Measure and Mobile Robots}
An early implementation of the ZMP measure to a mobile manipulator is presented in \cite{sugano1993}, and further work on mobile robots with stability constraint evolved from it. Guided by the ZMP formulation, a mobile manipulator's base is utilized to generate stability-compensating motions, while the manipulator arm is executing tasks \cite{huang2000,dine2018}. By using potential functions derived from the ZMP formulation, stability-compensating motion can also be generated for manipulator arms, as demonstrated in \cite{kim2002, choi2012, lee2012}. However, requiring a robot to continuously generate motion that compensates for stability during operations will unnecessarily reduce a robot's efficiency for task completion. 

The ZMP measure can also be employed as a constraint in the path following optimal control algorithms as demonstrated in \cite{mohammadi2016}. But for paths generated without paying attention to stability, applying the ZMP constraint at the path following stage can cause infeasibility when the optimal control formulation tries to find a solution. 


\subsection{Robotics in Timber Harvesting}
Timber harvesting, which is the targeted industry of application of our work, is a very important industry for many countries including Canada. The majority of modern timber harvesting businesses worldwide employ feller-bunchers and timber harvesters to fell trees, and use skidders and forwarders to transport felled trees for further processing. The aforementioned machines can be considered as a type of mobile manipulator since they consist of mobile bases and hydraulically powered mechanized arms with multiple degrees of freedom. Nowadays, the machines still fully rely on operator judgement and control in order to function \cite{lindroos2017}. The lack of machine automation in the timber harvesting industry results in high operator training costs and suboptimal efficiency in harvesting; there is a strong impetus for developing an autonomous harvesting system.

Some progress has been made to increase the timber harvesting machines' autonomy. The dynamics model of a timber harvester is presented in \cite{papa1997a}; the teleoperation of a forestry manipulator is showcased in \cite{wester2008}; the hydraulic actuator control of a forwarder machine has been discussed in \cite{hera2009} and \cite{morales2014}, and the motion control of a forestry manipulator along a fixed path is presented in \cite{morales2015}. However, a versatile trajectory planning algorithm that is tailored to the terrain related challenges of timber harvesting is yet to be developed. 

\subsection{Contributions}
Building on our earlier work in \cite{self2020} that mainly deals with stability constrained manipulation with simple mobile capability, the framework presented in this paper consists of the following major additions:
\begin{itemize}
\item{The stability-constrained path and trajectory planning framework for mobile manipulation on varying 3-D terrain is developed.}
\item{A modification to sampling-based path planning algorithms that guarantees the satisfaction of non-convex stability constraints in continuous time is developed.}
\item{A first theoretical analysis of the well-known ZMP stability measure with respect to motion on 3-D terrain with known topography is carried out.}
\end{itemize}

To the best of the authors' knowledge, the framework presented in this paper is the first to enable 3-D terrain mobile manipulation planning with a focus on task completion instead of stability compensation. In other words, a stability measure acts as a constraint that does not interfere with the robot's motion unless the constraint is violated. The main benefits of the proposed framework include higher task completion rate and higher motion efficiency compared to the previous work. This differs drastically from planning methods in \cite{dine2018, kim2002, choi2012, lee2012}, and \cite{huang2000} that mainly focus on generating motions that improve stability. 

\subsection{Organization}
In this paper, a feller buncher machine employed for timber harvesting is considered as an example of a mobile manipulator robot. In Section II, the kinematics model of the machine and the ZMP measure formulation are introduced as theoretical background for the paper. In Section III, the optimal control formulation of the trajectory planning problem is presented subject to the ZMP stability constraint. In light of the high dimensionality of the full model and the resulting computational complexity of solving the optimal trajectory planning problem, we propose a simplified model of the feller buncher; this model also allows us to obtain an analytical solution to the trajectory planning problem under certain assumptions. Simulation results showcasing the motion plan will be presented in Section IV. Discussion and potential future work will be mentioned in Section V.

\section{Theoretical Background}
This paper provides a solution to the problem of rough terrain optimal trajectory planning of mobile manipulators subject to dynamic stability constraint. Some theoretical background including the kinematics of a mobile manipulator robot, the mathematical definition of dynamic stability, and the general formulation of the problem are introduced in this section.

\subsection{Kinematics of Mobile Manipulator}
In order to accommodate the constraints imposed by the robot's joint properties and ZMP stability, we first derive the full kinematic model of a mobile manipulator whose example is shown schematically in Figure \ref{schematic1}. The manipulator arm of the robot has $n \in \mathbb{Z}^{0+}$ joint DoFs, and the mobile base is modeled as a unicycle on arbitrary terrain. The method presented in this paper is also applicable to car-like mobile robots without manipulators by treating them as mobile manipulators with $n=0$.

Frame $\mathcal{O}$ is the inertial frame, and frame $\mathcal{F}_i$ is fixed to the $i$-th link of the robot for $i\in\{0,1,\dots,n\}$. Here, link $0$ refers to the robot's mobile base, and links $1$ to $n-1$ refer to the manipulator arm components, and link $n$ refers to the robot's end effector. The three axes of each link frame $\boldsymbol{x}_i$, $\boldsymbol{y}_i$, and $\boldsymbol{z}_i$ are represented by red, green, and blue arrows, respectively.

\begin{figure}[h!]
\centering
\includegraphics[angle=0,origin=c,trim = 0mm 0mm 0mm 0mm, clip, width=8cm]{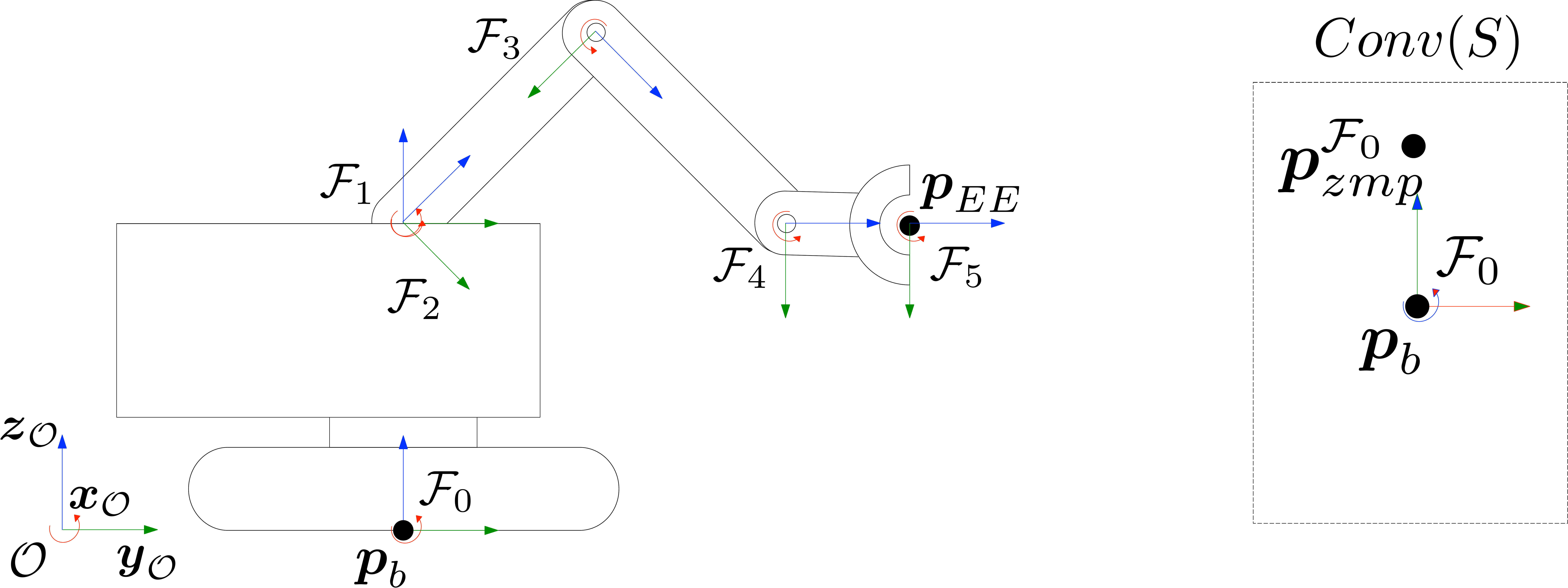}
\caption{Schematic diagram of a tracked mobile manipulator with 5 manipulator DoFs and its rectangular support polygon $Conv(S)$.}
\label{schematic1}
\end{figure}

We denote the angle of joint $i$ between link $i-1$ and $i$ with $q_i\in\mathbb{R}$ ($i = 1,2,\dots,n$), and collect all joint angles into a column vector $\boldsymbol{q}\in\mathbb{R}^n$, and consider the base-fixed frame $\mathcal{F}_0$ to undergo a $\boldsymbol{z}$-$\boldsymbol{x}$-$\boldsymbol{y}$ rotation from initial attitude with respective angles represented by $\psi$ (yaw), $\theta$ (pitch), and $\phi$ (roll). The Cartesian coordinates of each component's center of mass relative to $\boldsymbol{p}_b$ expressed in $\mathcal{F}_0$, $\boldsymbol{p}_i^{\mathcal{F}_0}$, can be found similarly and written succinctly by introducing $\bar{\boldsymbol{q}} = [\boldsymbol{p}_b^{\mathcal{O}^T}, \psi, \theta, \phi, \boldsymbol{q}^T]^T$ as:

\begin{equation}
\boldsymbol{p}_i^{\mathcal{F}_0} = \boldsymbol{f}_i(\bar{\boldsymbol{q}}).
\label{fi}
\end{equation}
From (\ref{fi}), we also obtain the inertial accelerations of links' centers of mass expressed in $\mathcal{F}_0$, which will be needed for the dynamics analysis later:
\begin{equation}
\begin{aligned}
\dot{\boldsymbol{p}}_i^{\mathcal{F}_0} &= \boldsymbol{J}_{pi}(\bar{\boldsymbol{q}})\dot{\bar{\boldsymbol{q}}},\\
\ddot{\boldsymbol{p}}_i^{\mathcal{F}_0} &= \dot{\boldsymbol{J}}_{pi}(\bar{\boldsymbol{q}})\dot{\bar{\boldsymbol{q}}}+\boldsymbol{J}_{pi}(\bar{\boldsymbol{q}})\ddot{\bar{\boldsymbol{q}}},
\end{aligned}
\label{ji}
\end{equation}
where we introduced the position kinematics Jacobian, $\boldsymbol{J}_{pi}(\bar{\boldsymbol{q}})=\frac{\partial{\boldsymbol{f}_i}}{\partial{\bar{\boldsymbol{q}}}}$. 


As noted earlier, we consider mobile base as a unicycle model and the robot is therefore subject to control input $\boldsymbol{u}=[u_a,u_\psi,\boldsymbol{u}_q^T]^T$, where $u_a$, $u_\psi$, and $\boldsymbol{u}_q$ represent acceleration along heading direction, yaw angular acceleration, and manipulator arm joint accelerations, respectively. Defining $\boldsymbol{x} = [\bar{\boldsymbol{q}}^T, \dot{\bar{\boldsymbol{q}}}^T]^T$, the mobile manipulator's kinematics equations can be written in state-space form as:
\begin{equation}
\dot{\boldsymbol{x}} = g(\boldsymbol{x},\boldsymbol{u}).
\label{kin_main}
\end{equation}
Eq. (\ref{kin_main}) serves as the model of the robot for the overall optimal trajectory planning problem. The terrain information is also considered to be embedded in (\ref{kin_main}).

In addition to the kinematics equation (\ref{kin_main}), a geometric construction called the support polygon will be used throughout this paper. The support polygon, denoted by $Conv(S)$, is a convex hull formed by the contact points between the robot's mobile base and the ground. A graphical illustration of the support polygon is shown in Figure \ref{schematic1}.

\subsection{Dynamic Stability of Mobile Manipulator}
The ZMP measure allows to quantify the rollover stability margin of a mobile manipulator on arbitrary terrain using only its kinematic model and inertia parameters, instead of full dynamics model or force measurements. It is therefore our method of choice due to the potential benefit of faster trajectory planning calculations and its practicality. 

According to \cite{sugano1993}, letting $\boldsymbol{p}_i^{\mathcal{F}_0} = [x_i,y_i,z_i]^T$, $\ddot{\boldsymbol{p}}_i^{\mathcal{F}_0} = [\ddot x_i,\ddot y_i,\ddot z_i]^T$, and the gravitational vector $\boldsymbol{g}^{\mathcal{F}_0} = [g_x,g_y,g_z]^T$, the coordinates of ZMP location in the base frame $\boldsymbol{p}_{\text{zmp}}^{\mathcal{F}_0} = [x_{\text{zmp}}, y_{\text{zmp}}, 0]^T$ are expressed as:
\begin{equation}
\begin{aligned}
x_\text{\text{zmp}} &= \dfrac{\sum_i m_i(\ddot{z}_i-g_z)x_i-\sum_i m_i(\ddot{x}_i-g_x) z_i}{\sum_i m_i(\ddot{z}_i-g_z)}\\
y_\text{\text{zmp}} &= \dfrac{\sum_i m_i(\ddot{z}_i-g_z)y_i-\sum_i m_i(\ddot{y}_i-g_y) z_i}{\sum_i m_i(\ddot{z}_i-g_z)},
\end{aligned}
~~\forall~~i
\label{zmp_origin}
\end{equation}
Note that differently from the formulations presented in previous works \cite{sugano1993}, the signs in front of the gravitational acceleration components are negative since gravity should be treated as an external force on the system. Also, by definition, ZMP lies in the plane of the support polygon so that its z-coordinate in $\mathcal{F}_0$ is zero. The mass of the manipulated object and the terrain topography are assumed to be known in this work and can be accounted for in (\ref{zmp_origin}). With $Conv(S)$ representing the mobile manipulator's support polygon, the robot is dynamically stable when $\boldsymbol{p}_{\text{zmp}}\in Conv(S)$, but has the tendency to rollover otherwise. Equation (\ref{zmp_origin}) can be evaluated by substituting from the kinematics equations (\ref{fi}) and (\ref{ji}) represented in $\mathcal{F}_0$, once the joint motions have been planned, to produce the ZMP locus, i.e., the trajectory of ZMP in $\mathcal{F}_0$ during the motion of the robot.

\subsection{Optimal Control Problem Formulation}
We consider this problem in full generality by allowing the robot to be situated on slopes of different grades. With the view to optimizing the efficiency of manipulation tasks, the motions should be carried out within the shortest possible time, without the robot rolling over. Naturally, to minimize the overall time required for the robot to reach final state $\boldsymbol{x}_f$ from an initial state $\boldsymbol{x}_0$ safely, we formulate a constrained nonlinear optimal control problem (OCP) of the form:
\begin{equation}
\begin{aligned}
\min_{\boldsymbol{u}} ~~&\int_{t_0}^{t_f} 1~ \text{d} t.\\
\text{s.t.}~~& \dot{\boldsymbol{x}}=g(\boldsymbol{x},\boldsymbol{u})~~\boldsymbol{x}(t_0) = \boldsymbol{x}_0~~\boldsymbol{x}(t_f) = \boldsymbol{x}_f\\
&\boldsymbol{p}_{\text{zmp}}(\boldsymbol{x},\boldsymbol{u})\in Conv(S)\\
&\underline{\boldsymbol{x}} \preccurlyeq\boldsymbol{x}\preccurlyeq\overline{\boldsymbol{x}}\\
&\underline{\boldsymbol{u}} \preccurlyeq\boldsymbol{u}\preccurlyeq\overline{\boldsymbol{u}},
\end{aligned}
\label{opt_formulation}
\end{equation}
where $\preccurlyeq$ is defined as vector component-wise inequality, $\underline{\boldsymbol{x}}$ and $\overline{\boldsymbol{x}}$ stand for the lower and upper bounds on the system state $\boldsymbol{x}$, and $\underline{\boldsymbol{u}}$ and $\overline{\boldsymbol{u}}$ stand for the lower and upper bounds on the control input $\boldsymbol{u}$, respectively. The state and input constraints ensure that the robot's configuration, joint rates, and joint accelerations are realistic and therefore feasible when tracked. 

The optimal control problem defined in (\ref{opt_formulation}) can be solved by using a nonlinear optimal control solver such as GPOPS \cite{patt2014}, but the computational time required is too long for online motion guidance. The ZMP stability constraint (\ref{zmp_origin}) is an important contributor to high computational cost as it is nonlinear and non-convex with respect to the system's states, including the robot's joint angles and orientation.

\section{Local Manipulation Planning Under Dynamic Stability Constraint}
In most manipulation scenarios, a common task for a mobile manipulator is to interact with objects by securing, lifting, and placing them at a location either nearby or at a distance. When dissected into two general modes of operation -- mobile relocation and local manipulation -- any mobile manipulation task can be defined as a combination of these two modes, either taking place simultaneously or separately. Therefore, it is desirable to solve the two stages separately in order to avoid the computational burden caused by a high-dimensional nonlinear OCP.

The \textbf{local manipulation} planning of the mobile manipulator robot is addressed in this section. During this stage, the robot has to execute joint motions in order to reconfigure itself or interact with an object. The dynamic stability of the robot is likely to be compromised when the robot has high center of mass, when the joint motions are aggressive, when the terrain is steep, and when the object is heavy. Trajectory planning for the mobile relocation stage will be introduced in Section \ref{relocation}.
\label{manipulation}

\subsection{Simplified Model of Mobile Manipulator with Reduced Mobile Capability}
Due to the high dimensionality of the kinematics model (\ref{kin_main}), the nonlinearity that resides in the forward kinematics equations (\ref{fi}) and the ZMP formulation (\ref{zmp_origin}), the solution of the optimal manipulation planning problem (\ref{opt_formulation}) requires longer than practical computation time for online guidance, even for a case where the base does not move. In order to develop a online trajectory planner for the robot, a model simplification is called upon. 

\begin{figure}[h!]
\centering
\includegraphics[angle=0,origin=c,trim = 0mm 0mm 0mm 0mm, clip, width=8cm]{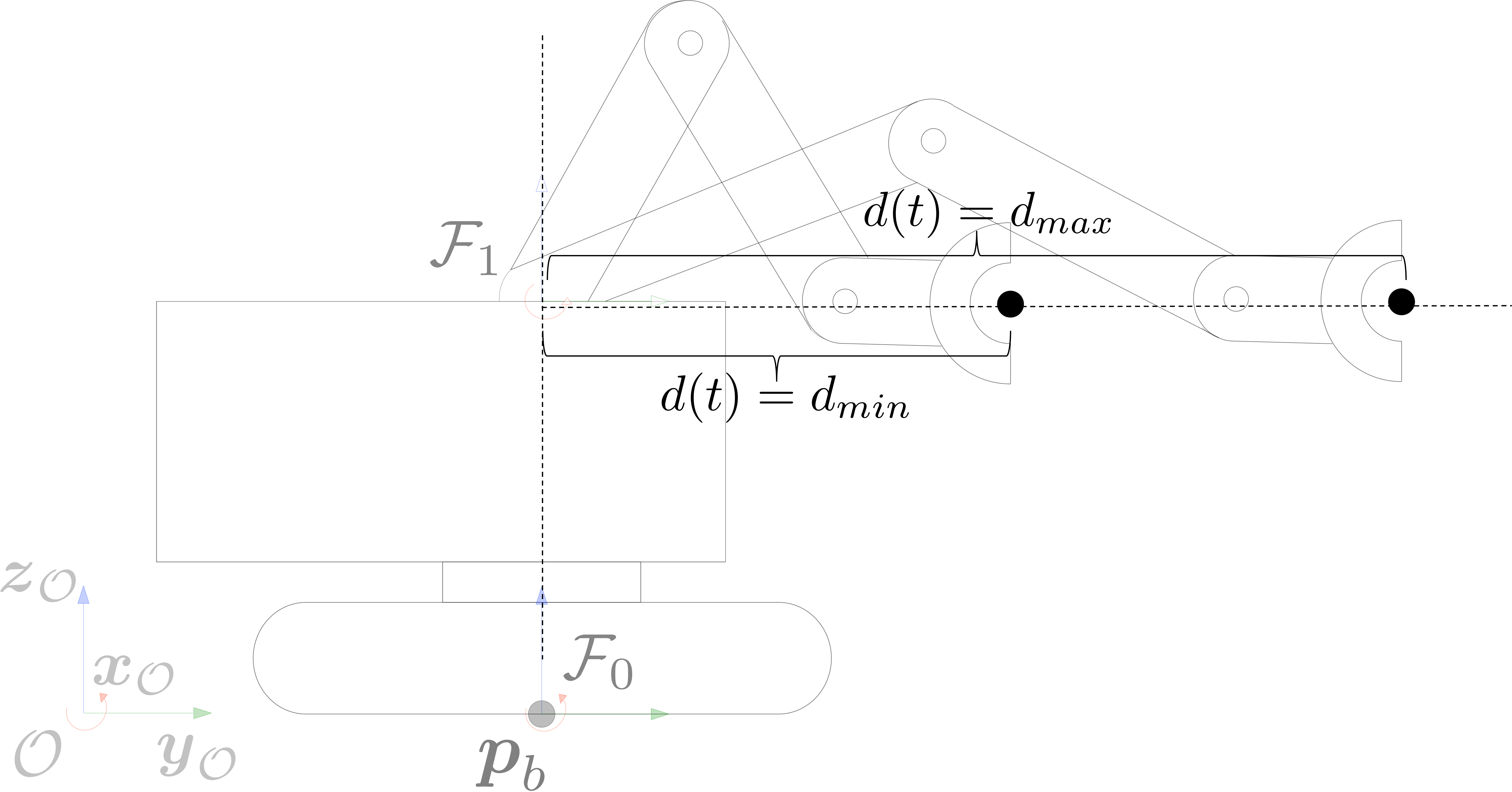}
\caption{Schematic diagram of a simplified mobile manipulator.}
\label{heli_cart}
\end{figure}

To simplify the optimal control problem by reducing the dimensionality of the problem, and considering typical modes of operation such as of a feller-buncher in timber harvesting or excavators in mining and construction, the $n$-joint-DoF manipulator arm is simplified to have 2 DoFs including the first link yaw angle $q_1(t)$ and variable end effector distance $d(t)$ from the yaw joint axis. The simplified model is illustrated schematically in Figure \ref{heli_cart}. The simplification is also based on the following assumptions: 
\begin{assumption1}
a) The mobile base has a fixed attitude and accelerates along $\boldsymbol{y}_0$ direction.
b) The $z$-coordinate of end effector does not change in $\mathcal{F}_0$. 
c) The inverse kinematics mapping between end effector distance $d$ and full model arm joint angles $\boldsymbol{q}$ except for $q_1$ is known up to the second order derivative with respect to time.
\end{assumption1}
Note that in practice, local manipulation tasks such as picking and placing rarely require complex movements of the mobile base, Assumption 1a) can be justified for cases such as excavation and timber harvesting. A formulation that allows the execution of complex mobile maneuvers is presented in Sections \ref{relocation} in the context of a relocation problem. To mitigate the limitations caused by Assumption 1b), a specially designed motion planner can manipulate the end effector into or out of the $x$-$y$ plane of $\mathcal{F}_0$, without significant loss in performance as the vertical movements required are usually small. Assumption 1c) is easily satisfied as inverse kinematics resolution method mentioned in \cite{reiter2018} is well established.

In the simplified model, the control inputs of the simplified model are $u_{q_1}$, acceleration of $q_1$, and $u_d$, the second derivative of $d$ with respect to time. The kinematics model of the simplified model is straightforward to derive and can be written in state space form:

\begin{equation}
\dot{\tilde{\boldsymbol{x}}} = \tilde{g}(\tilde{\boldsymbol{x}},\tilde{\boldsymbol{u}}),
\label{kin_simplified}
\end{equation}
where $\tilde{\boldsymbol{x}} = [q_1,\dot{q}_1,d,\dot{d}]^T$, and $\tilde{\boldsymbol{u}} = [u_{q_1}, u_d]$.

\subsection{Mapping from Simplified Model to Full Model}
To ensure that the ZMP location of the full model follows that of the simplified model, a mapping between the DoFs of the simplified model and those of the full model of the robot needs to be established. From Assumption 1c), the robot's joint angles, angular rates, and accelerations can be expressed using the states and control inputs of the simplified model as:
\begin{equation}
\begin{aligned}
\boldsymbol{q}&=[q_1, \mathcal{IK}(d)]^T\\
\boldsymbol{\dot{q}}&=[\dot{q}_1, \mathcal{IK}^1(\dot{d},d)]^T\\
\boldsymbol{\ddot{q}}&=[u_{q_1}, \mathcal{IK}^2(u_d,\dot{d},d)]^T.
\end{aligned}
\label{IK_map}
\end{equation}
Here, $\mathcal{IK}$ represents the inverse kinematics mapping referred to in Assumption 1c) equation and the superscript signifies its order of derivative with respect to time. Therefore, using the mapping provided in (\ref{kin_main}), (\ref{zmp_origin}), and (\ref{IK_map}), the ZMP location of the full kinematics model can be expressed using the simplified kinematic quantities $\tilde{\boldsymbol{x}}$ and $\tilde{\boldsymbol{u}}$ as $\tilde{\boldsymbol{p}}_{\text{zmp}}(\tilde{\boldsymbol{x}},\tilde{\boldsymbol{u}})$.

For certain limited scenarios, a further physical interpretation of the ZMP measure that allows the analytical derivation of a bang-bang type time optimal local manipulation plan is shown in our previous work in \cite{self2020}.

\subsection{Simplified Optimal Control Problem Formulation}
\subsubsection{Manipulation Planning}
To shorten the computational time so that the trajectory planner can provide online guidance, the optimal control problem (\ref{opt_formulation}) can be reformulated for the simplified model (\ref{kin_simplified}) as:
\begin{equation}
\begin{aligned}
\min_{\tilde{\boldsymbol{u}}} ~~&\int_{t_0}^{t_f} 1~ \text{d} t.\\
\text{s.t.}~~& \dot{\tilde{\boldsymbol{x}}}=\tilde{g}(\tilde{\boldsymbol{x}},\tilde{\boldsymbol{u}})~~\tilde{\boldsymbol{x}}(t_0) = \tilde{\boldsymbol{x}}_0~~\tilde{\boldsymbol{x}}(t_f) = \tilde{\boldsymbol{x}}_f\\
&\tilde{\boldsymbol{p}}_{\text{zmp}}(\tilde{\boldsymbol{x}},\tilde{\boldsymbol{u}})\in Conv(S)\\
&\underline{\tilde{\boldsymbol{x}}} \preccurlyeq\tilde{\boldsymbol{x}}\preccurlyeq\overline{\tilde{\boldsymbol{x}}}\\
&\underline{\tilde{\boldsymbol{u}}} \preccurlyeq\tilde{\boldsymbol{u}}\preccurlyeq\overline{\tilde{\boldsymbol{u}}}.
\end{aligned}
\label{opt_reformulation}
\end{equation}
It is noted that although the number of variables has been reduced, the ZMP constraint employed in (\ref{opt_reformulation}) is still derived from (\ref{zmp_origin}) using information of all joints obtained from (\ref{IK_map}). Hence, the motions generated by solving the OCP (\ref{opt_reformulation}) will guarantee ZMP constraint satisfaction for the full model. Also, we point out that in replacing the optimization problem (\ref{opt_formulation}) with (\ref{opt_reformulation}), the dimension of the OCP kinematic constraints has been reduced from $2n+12$ to $4$, and the number of control inputs has been reduced from $n+2$ to 2. However, the input constraints can no longer be placed directly on joint accelerations, although in practice, this can be resolved by tightening the bound on $\bar{\boldsymbol{u}}$ until all joint accelerations are feasible.

\subsubsection{Manipulation-coordinated Base Acceleration Planning}
Since the configuration space of the simplified model does not include the coordinates of the mobile base, planning of the base motion must be treated separately. Under Assumption 1, we have $\ddot{\boldsymbol{p}}_0^{\mathcal{F}_0}=[0,\ddot{y}_0,0]^T$, and therefore, by taking the partial derivative of $x_{\text{zmp}}$ and $y_{\text{zmp}}$ with respect to $\ddot{y}_0$, we can see that $x_{\text{zmp}}$ does not change when the base accelerates, and the variation in $y_{\text{zmp}}$ can be described by:
\begin{equation*}
\frac{\partial y_{\text{zmp}}}{\partial \ddot{y}_0} = \frac{-\sum_i m_iz_i}{\sum_i m_i(\ddot{z}_i-g_z)}.
\end{equation*}

In common manipulation operations, it is reasonable to assume that $\ddot{z}_i$'s are negligible and $z_i$'s are constants. Once  the optimal control problem is solved using the formulation (\ref{opt_reformulation}), the margin between the upper and lower edges of the support polygon to the planned ZMP trajectory, $d_u$ and $d_l$, respectively (see Figure \ref{planvbang}), can be found, and upper and lower bounds on the mobile base's linear acceleration can be derived as:
\begin{equation}
d_u\frac{\sum_i m_ig_z}{\sum_i m_iz_i}\leq \ddot{y}_0 \leq -d_l\frac{\sum_i m_ig_z}{\sum_i m_iz_i}.
\label{acc_bound}
\end{equation}
Acceleration of the mobile base can then be planned using bang-bang type control with the above as input constraint. 


\section{Rough Terrain Relocation Planning}
In this section, trajectory planning of the \textbf{relocation} stage of mobile manipulation is addressed. During relocation, a mobile manipulator has to navigate a constantly undulating terrain and, sometimes, shift its ZMP through reconfiguring the manipulator to safely move through steep slope. The topographic map of the 3D terrain that contains the robot and goal point is assumed to be known. A sampling-based path and reconfiguration planning algorithm with dynamic stability guarantee will now be introduced. The result from path planning will then be used as an initial guess to warm-start a nonlinear optimal control problem for fast convergence.
\label{relocation}


\begin{algorithm*}[h!]
\caption{ZMP-constrained Path Planning}
\begin{algorithmic}[1]
\Procedure{TerrainRRT}{$p_{init},p_{goal},q_{robot}$}
	\State $V \gets p_{init}$
	\State $E \gets \{\}$
	
	\State $Q \gets q_{robot}$
	\For{$i=1,2,\ldots, n$}
		\State $p_{rand} \gets SampleFree(i)$
		\State $p_{nearest} \gets Nearest(V, p_{rand})$
		\State $p_{node} \gets Steer(p_{nearest},p_{rand})$
		\If {ObstacleFree($p_{nearest}, p_{node}$) \textbf{and} \textcolor{blue}{ZMPStable($p_{nearest}, p_{node}, q_{robot}$)}} 
        		\State $V \gets V\cup \{p_{node}\}$
        		\State $E \gets E\cup \{(p_{nearest}, p_{node})\}$
        		\State $Q \gets Q\cup \{q_{robot}\}$
		\EndIf
		\textcolor{red}{
		\If{TreeGrowth($V,i$) $\leq$ thereshold} 
			\State $q_{robot} \gets Reconfigure(q_{robot})$ \Comment{Change robot configuration when size of $V$ grows slowly.}
		\EndIf}
	\EndFor
	\Return{$G = (V,E)$}
\EndProcedure
\end{algorithmic}
\label{TRRT}
\end{algorithm*}

\subsection{Quasi-static Path Planning Algorithm}
Sampling-based approach is chosen to solve the path planning problem due to its theoretical completeness, ability to deal with nonlinear constraints, and efficiency when search space is purposefully constructed. The sampling based planning structure is based on the RRT algorithm proposed in \cite{lavalle2001}. The pseudo-code that illustrates the procedure of path planning with ZMP constraint is shown in Algorithm \ref{TRRT}.  The modification made to accommodate stability constraint is highlighted in \textcolor{blue}{blue}, and the modification made to allow reconfiguration is highlighted in \textcolor{red}{red}.

The RRT algorithm, specifically, is chosen due to its ability to quickly find a path when new terrain features and obstacles such as human, fallen trees, or large rocks appear on the map. However, the ZMP constraint accommodation can also be directly implemented within other sampling based planning algorithms. 

The robot is considered to be quasi-statically moving through the terrain in order to remove its linear and angular acceleration from the search space. The robot only reconfigures its joint angles when the tree growth is slow. As a result, the search space dimension is further reduced on mild terrain, and the path searching time can be drastically reduced.

\subsubsection{Stability Guarantee}
To ensure the sampling based algorithm can generate a quasi-static path that is guaranteed to be ZMP-stable in continuous execution, the $\text{\textcolor{blue}{ZMPStable}}$ function on line 9 of Algorithm \ref{TRRT} checks that the stability conditions to be introduced in Section \ref{sta_guaran} are met.

\subsubsection{Reconfiguration}
As the path planning tree growth reaches the steep section in a terrain map, new nodes will become less likely to be added to the tree. This is due to stability constraints becoming harder to satisfy. To remedy this problem, random joint reconfiguration is executed whenever the number of new nodes added to the tree is lower than a threshold after certain number of iterations. In this way, the robot will be able to navigate through steep slopes by reconfiguring its arm. The method introduced in Section \ref{manipulation} is used to further plan the motion trajectory of each reconfiguration.

\subsection{Finite-time Feasibility Guarantee}
Due to the quasi-static nature of the path planning stage, the inertial terms in the ZMP constraints are ignored. Therefore, a quasi-static path generated with Algorithm \ref{TRRT} does not guarantee that a robot can travel to its destination without violating the dynamic stability constraint when the robot's components experience non-zero linear and angular accelerations. In order to address this issue, the following theorem is introduced:
\begin{theorem}
Along a path $\tau: I \to \boldsymbol{x}$, where $I=[0,1]$ is a unit interval and the robot's states $\boldsymbol{x}$ are quasi-static with $\dot{\boldsymbol{x}}=0$, if the ZMP location $\boldsymbol{p}_{zmp}(\tau)$ belongs to the interior of $Conv(S)~\forall~\tau$, $\exists$ $\boldsymbol{u}$ for the trajectory planning problem (\ref{opt_formulation}) such that $t_f=T<<\infty$.
\label{t1}
\end{theorem}
A proof of Theorem \ref{t1} is provided in the Appendix section. With Theorem \ref{t1} backing, the method to find a quasi-static path that satisfies the ZMP constraint is presented in the following subsections.

\subsection{Stability Guaranteed Sampling-based Path Search for Continuous Motion}
\label{sta_guaran}
To plan a robot path on 3-D terrain using a sampling-based algorithm, such as RRT, nonlinear constraints are checked at each sampling point to ensure constraint violations do not occur. However, to keep the computation time reasonable, constraint satisfaction cannot be verified continuously in between sampling points. Since the ZMP constraint (\ref{zmp_origin}) is non-convex with respect to the robot's base orientation, even if the robot is dynamically stable at two neighboring sampling points, it is not guaranteed to be dynamically stable between two sampling points. To resolve this issue, an analytical method to determine the ZMP stability between sampling points is desired.

In the stability constrained RRT implementation, ZMP constraint satisfaction is unchecked in the following two scenarios. The first is when the robot's base changes its heading direction at a particular sampling location, and the second is as the robot quasi-statically ``slides" from one sampling location to the next. Since the two instances happen separately in the RRT implementation, they can be dealt with individually. In the following analysis, we will make use of $g_z>0$, since we do not expect the robot to drive up vertical walls. In addition, under the quasi-static condition, (\ref{zmp_origin}) reduces to
\begin{equation}
x_\text{zmp} = \dfrac{M_x g_z-M_z g_x}{M g_z},~~
y_\text{zmp} = \dfrac{M_y g_z-M_z g_y}{M g_z},
\label{zmp_slide}
\end{equation}
where $M_x=\sum_i m_i x_i$, $M_y=\sum_i m_i y_i$, $M_z=\sum_i m_i z_i$, and $M=\sum_i m_i$. When the same robot is located on the horizontal terrain, according to (\ref{zmp_slide}), its ZMP location is further simplified to:
\begin{equation}
        x_{\text{zmp,0}}=\frac{M_x}{M},~~
        y_{\text{zmp,0}}=\frac{M_y}{M}.
        \label{zmp0}
\end{equation}

\subsubsection{Mobile Base Attitude}
To analyze the ZMP stability of the two aforementioned scenarios, we first need to define the robot base attitude on arbitrary 3-D terrain. With $\boldsymbol{p}_b=[s_x, s_y, s_z]^T$ representing the 3-D coordinate of terrain surface in inertial frame $\mathcal{O}$ at the location of the mobile base, the terrain can be expressed by the following nonlinear function:
\begin{equation}
    s_z = h(s_x, s_y).
\end{equation}

Let us denote the rotation matrix that describes the attitude change of the base-fixed frame $\mathcal{F}_0$ from the inertial frame $\mathcal{O}$ as $\boldsymbol{R}=\{r_{ij}\}~(i,j\!\!=\!\!1,2,3)$, and since the base-fixed $z_0$ axis of a mobile manipulator remains perpendicular to terrain surface under normal operations, defining the directional gradients of the terrain $\boldsymbol{h}_x = \left[1,0,\frac{\partial h}{\partial s_x}\Bigr|_{s_x,s_y}\right]^T$ and $\boldsymbol{h}_y = \left[0,1,\frac{\partial h}{\partial s_y}\Bigr|_{s_x,s_y}\right]^T$, the third column of $\boldsymbol{R}$ can be represented as:
\begin{equation*}
    \hat{\boldsymbol{r}}_3 = \frac{\boldsymbol{h}_x \times \boldsymbol{h}_y}{\left|\left| \boldsymbol{h}_x \times \boldsymbol{h}_y \right| \right|}.
\end{equation*}

Let us denote the heading angle of the mobile base projected onto the 2-D plane $\boldsymbol{x}_{\mathcal{O}}$-$\boldsymbol{y}_{\mathcal{O}}$ as $\bar{\psi}$. The second column of $\boldsymbol{R}$ describes the mobile base's heading direction on the 3-D surface and can then be written as:
\begin{equation*}
    \hat{\boldsymbol{r}}_2 = \frac{[\boldsymbol{h}_x, \boldsymbol{h}_y][-\sin{\bar{\psi}}, \cos{\bar{\psi}}]^T}{\left|\left|[\boldsymbol{h}_x, \boldsymbol{h}_y][-\sin{\bar{\psi}}, \cos{\bar{\psi}}]^T\right|\right|}.
\end{equation*}
Hence, the first column of $\boldsymbol{R}$ can be found using:
\begin{equation*}
    \hat{\boldsymbol{r}}_1 = \frac{ \hat{\boldsymbol{r}}_2 \times \hat{\boldsymbol{r}}_3}{\left|\left| \hat{\boldsymbol{r}}_2 \times \hat{\boldsymbol{r}}_3\right|\right|}.
\end{equation*}

\subsubsection{Stability Guarantee During Base Turn}
Inspecting (\ref{zmp_slide}), it is clear that $g_x$, $g_y$, and $g_z$ are the only quantities that vary with the robot's orientation. As a robot whose attitude is described by the rotation matrix $\boldsymbol{R}$ executes a skid-steered turn through angle $\psi^t$, the gravity components are affected by the yaw angle $\psi^t$ and can be expressed as:
\begin{equation}
    \begin{aligned}
        \begin{bmatrix}
            g^t_x\\
            g^t_y\\
            g^t_z
        \end{bmatrix}&=
        \begin{bmatrix}
            \cos{\psi^t}&\sin{\psi^t}&0\\
            -\sin{\psi^t}&\cos{\psi^t}&0\\
            0&0&1
        \end{bmatrix}
        \underbrace{\begin{bmatrix}
            \hat{\boldsymbol{r}}_1^T\\
            \hat{\boldsymbol{r}}_2^T\\
            \hat{\boldsymbol{r}}_3^T
        \end{bmatrix}}_{\boldsymbol{R}^T}
        \begin{bmatrix}
            0\\0\\g
        \end{bmatrix}\\
        &=\begin{bmatrix}
            r_{13}\cos{\psi^t}g+r_{23}\sin{\psi^t}g\\
            -r_{13}\sin{\psi^t}g+r_{23}\cos{\psi^t}g\\
            r_{33}g
        \end{bmatrix}.
    \end{aligned}
    \label{g_comp}
\end{equation}
The ZMP after the turn is then found from:
\begin{equation}
x^t_\text{zmp} = \dfrac{M_x g^t_z-M_z g^t_x}{M g^t_z},~
y^t_\text{zmp} = \dfrac{M_y g^t_z-M_z g^t_y}{M g^t_z}.
\label{zmp_afturn}
\end{equation}

To aid in the illustration of stability guarantee during a skid-steered turn, the following claim is presented with a proof that follows:
\begin{proposition}
As a wheeled/tracked robot executes a continuous unidirectional quasi-static turn through angle $\psi^t$ along its yaw axis, the ZMP of the robot traverses, monotonically on the support polygon, an arc of angle $\psi^t$, radius $\frac{r_{13}^2M_z^2+r_{23}^2M_z^2}{r_{33}^2M^2}$, and center $[x_{\text{zmp,0}},~y_{\text{zmp,0}}]^T$.
\label{prop1}
\end{proposition}
\begin{proof}
By defining $\Delta x_{\text{zmp}}=x^t_{\text{zmp}}-x_{\text{zmp,0}}$, and $\Delta y_{\text{zmp}}=y^t_{\text{zmp}}-y_{\text{zmp,0}}$, the difference between (\ref{zmp_afturn}) and (\ref{zmp0}) can be represented as:
\begin{equation}
    \Delta x_{\text{zmp}} = \frac{M_zg^t_x}{-Mg^t_z}, ~~
    \Delta y_{\text{zmp}} = \frac{M_zg^t_y}{-Mg^t_z}.
    \label{dzmp}
\end{equation}
Substituting (\ref{g_comp}) into (\ref{dzmp}), we obtain
\begin{equation}
    \begin{bmatrix}
         \Delta x_{\text{zmp}}\\
         \Delta y_{\text{zmp}}
    \end{bmatrix} =
    \begin{bmatrix}
        \cos{\psi^t} & \sin{\psi^t}\\
        -\sin{\psi^t} & \cos{\psi^t}
    \end{bmatrix}
    \begin{bmatrix}
        \frac{r_{13}M_z}{-r_{33}M}\\
        \frac{r_{23}M_z}{-r_{33}M}
    \end{bmatrix}.
\end{equation}
The above represents a rotation in the $x_0$-$y_0$ plane of ZMP location by angle $\psi^t$ with respect to $[x_{\text{zmp,0}},~y_{\text{zmp,0}}]^T$.
\end{proof}

With the help of Proposition \ref{prop1}, graphically illustrated in Figure \ref{turn_traj}, the ZMP stability of a robot during a quasi-static turn can be checked analytically for a given initial base attitude, arm configuration, and turn angle. As a result, the path planning algorithm is able to generate turns that are guaranteed to be continuously ZMP-stable.

\begin{figure}[h!]
\centering
\captionsetup[subfigure]{width=0.9\textwidth,justification=raggedright}
\begin{subfigure}{.25\textwidth}
\centering
\includegraphics[angle=0,origin=c,trim = 0mm 0mm 0mm 0mm, clip, width=3cm]{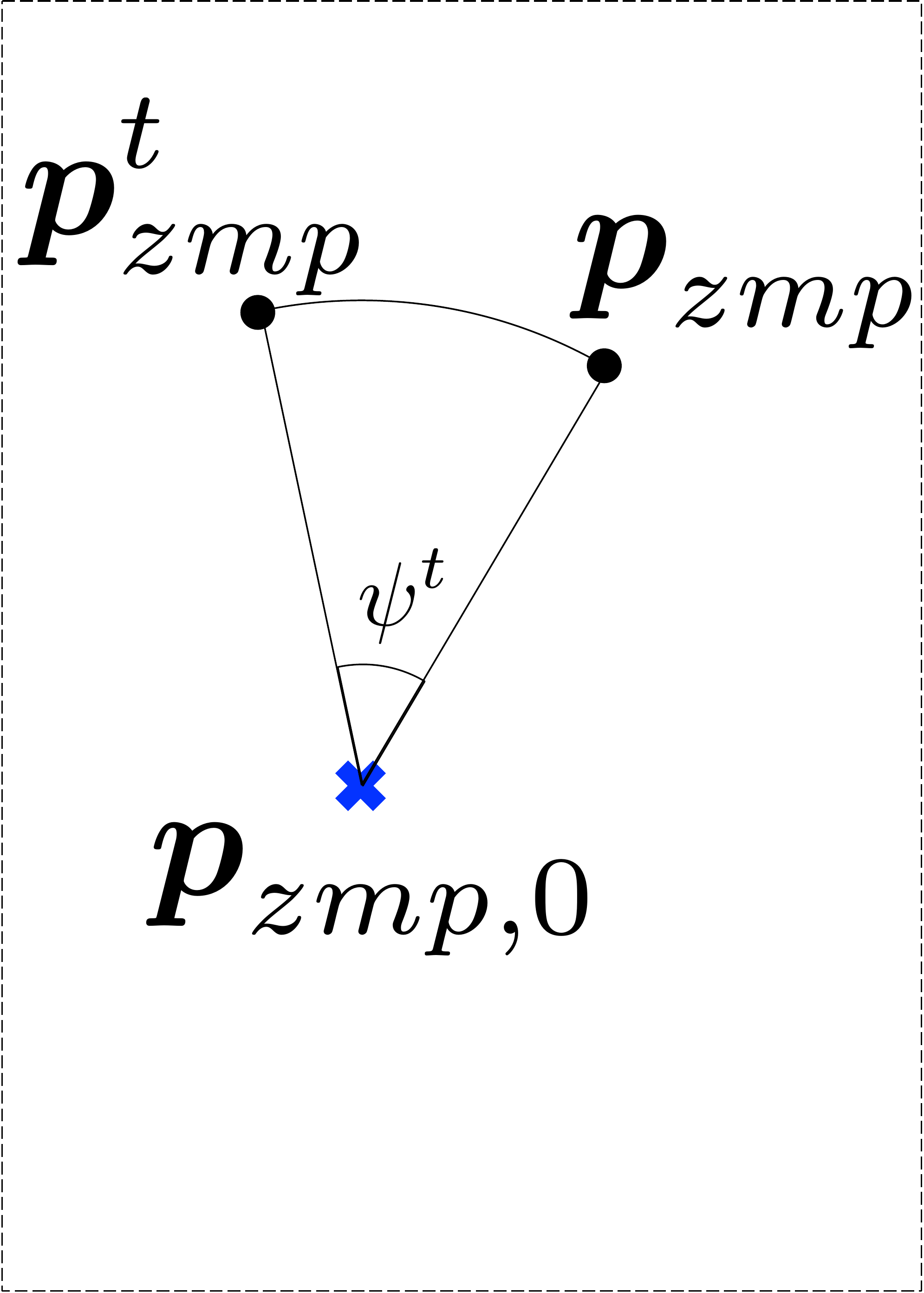}
\caption{Arc formed by ZMP trajectory during a turn.}
\label{turn_traj}
\end{subfigure}%
\begin{subfigure}{.25\textwidth}
\centering
\includegraphics[angle=0,origin=c,trim = 0mm 0mm 0mm 0mm, clip, width=3cm]{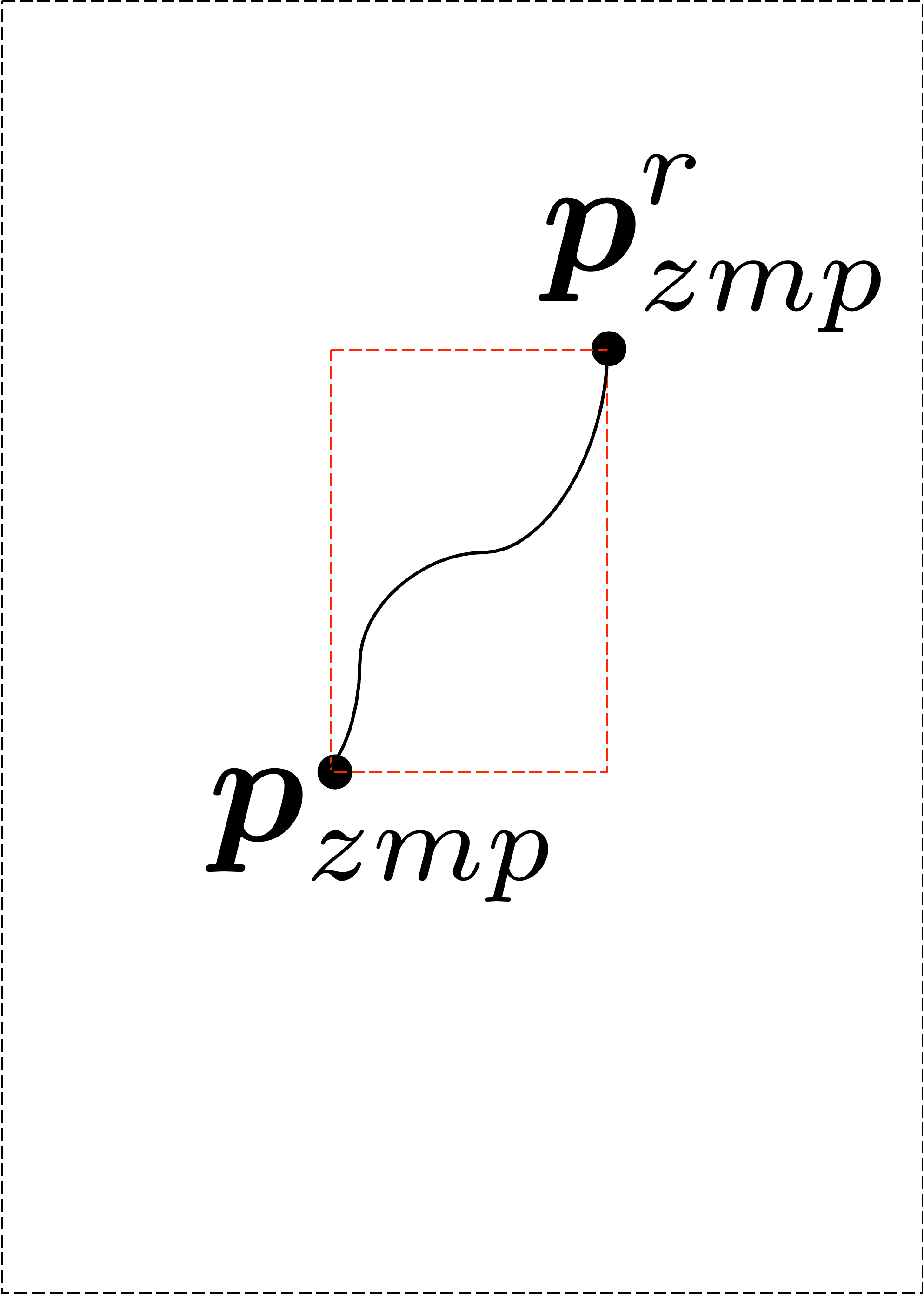}
\caption{ZMP trajectory and its rectangular bound (red) during relocation.}
\label{move_traj}
\end{subfigure}%
\caption{Graphical illustration of ZMP trajectory during robot base turn and mobile relocation.}
\label{change_traj}
\end{figure}

\subsubsection{Stability Guarantee During Base Relocation}
During base relocation between two neighboring sampling points, it is reasonable to represent the terrain-induced attitude change of the robot's base as a pitch-roll sequence. By restricting the distance between two sampling points to be ``small" during sampling point generation, we expect the changes of pitch and roll angle to be small and monotonic. With pitch and roll angle changes from the given base orientation $\boldsymbol{R}$ represented by $\theta^r$ and $\phi^r$, respectively, the perturbed gravity components can be written as:
\begin{equation}
    \begin{aligned}
        \begin{bmatrix}
            g^r_x\\
            g^r_y\\
            g^r_z
        \end{bmatrix}&\!\!=\!\!
        \begin{bmatrix}
            \cos{\phi^r} & \sin{\phi^r}\sin{\theta^r} & -\sin{\phi^r}\cos{\theta^r}\\
            0 & \cos{\theta^r} & \sin{\theta^r}\\
            \sin{\phi^r} & -\cos{\phi^r}\sin{\theta^r} & \cos{\phi^r}\cos{\theta^r}
        \end{bmatrix}\!\!\!
        \underbrace{\begin{bmatrix}
            \hat{\boldsymbol{r}}_1^T\\
            \hat{\boldsymbol{r}}_2^T\\
            \hat{\boldsymbol{r}}_3^T
        \end{bmatrix}}_{\boldsymbol{R}^T}\!\!\!
        \begin{bmatrix}
            0\\0\\g
        \end{bmatrix}\\&\!\!=\!\!
        \begin{bmatrix}
            r_{13}\!\cos{\phi^r}\!g\!+\!r_{23}\sin{\phi^r}\sin{\theta^r}g\!-\!r_{33}\sin{\phi^r}\cos{\theta^r}\!g\\
            r_{23}\cos{\theta^r}g+r_{33}\sin{\theta^r}g\\
            r_{13}\!\sin{\phi^r}\!g\!-\!r_{23}\cos{\phi^r}\sin{\theta^r}g\!+\!r_{33}\cos{\phi^r}\cos{\theta^r}\!g
        \end{bmatrix}\!\!.
    \end{aligned}
    \label{g_comp2}
\end{equation}
Substituting (\ref{g_comp2}) into (\ref{zmp_slide}), the following is obtained to describe the ZMP location as a function of pitch and roll angles due to relocation:
\begin{equation}
    \begin{aligned}
        x^r_{\text{zmp}}&=\frac{M_x}{M}-\frac{M_z}{M}G_1\\
        y^r_{\text{zmp}}&=\frac{M_y}{M}-\frac{M_z}{M}G_2,
    \end{aligned}
    \label{zmp_comp2}
\end{equation}
where
\begin{equation}
    \begin{aligned}
        G_1 &= \frac{r_{13}\cos{\phi^r}+r_{23}\sin{\phi^r}\sin{\theta^r}-r_{33}\sin{\phi^r}\cos{\theta^r}}{r_{13}\sin{\phi^r}-r_{23}\cos{\phi^r}\sin{\theta^r}+r_{33}\cos{\theta^r}\cos{\phi^r}}\\
        G_2 &= \frac{r_{23}\cos{\theta^r}+r_{33}\sin{\theta^r}}{r_{13}\sin{\phi^r}-r_{23}\cos{\phi^r}\sin{\theta^r}+r_{33}\cos{\theta^r}\cos{\phi^r}}.
    \end{aligned}
\end{equation}

The stability guarantee during base relocation between neighboring sampling points is presented as the following proposition with a proof:
\begin{proposition}
As a wheeled/tracked robot relocates from one sampling point to the next through a continuous motion, the ZMP locations are guaranteed to be bounded by a rectangle whose edges are parallel and perpendicular to the robot's heading direction; as well, the rectangle's non-neighboring corners are the ZMPs at the two sampling points that the robot is transitioning between.
\label{prop2}
\end{proposition}

\begin{proof}
To find the variation in ZMP location with respect to $\phi^r$ and $\theta^r$ each, we take the corresponding partial derivatives of $G_1$ and $G_2$. Furthermore, with the previously stated assumption that the distances between neighboring sampling points are small, we evaluate the partial derivatives at zero angles $\phi^r\approx 0$ and $\theta^r\approx 0$, and obtain:
\begin{equation}
    \begin{aligned}
        \frac{\partial G_1}{\partial \phi^r} &= -1-\frac{r_{13}^2}{r_{33}^2}\\
        \frac{\partial G_1}{\partial \theta^r} &= -\frac{r_{13}r_{23}}{r_{33}^2}\\
        \frac{\partial G_2}{\partial \phi^r} &= -\frac{r_{13}r_{23}}{r_{33}^2}\\
        \frac{\partial G_2}{\partial \theta^r} &= 1+\frac{r_{23}^2}{r_{33}^2}.
    \end{aligned}
    \label{pds}
\end{equation}
Recalling the orthonormality of the rotation matrix so that $r_{13}^2+r_{23}^2+r_{33}^2=1$, the following can be inferred from (\ref{pds}):
\begin{equation}
    \begin{aligned}
        \frac{\partial G_1}{\partial \phi^r} &\leq-1,~~\frac{\partial G_1}{\partial \theta^r} \approx 0\\
        \frac{\partial G_2}{\partial \phi^r} & \approx 0,~~~~\frac{\partial G_2}{\partial \theta^r} \geq 1.
    \end{aligned}
    \label{g1g2_pds}
\end{equation}
By substituting (\ref{g1g2_pds}) into (\ref{zmp_comp2}), we obtain the following constraints on the variation of ZMP coordinates:
\begin{equation}
    \begin{aligned}
        \frac{\partial x^r_{\text{zmp}}}{\partial {\phi^r}} &\geq \frac{M_z}{M}>0,~~\frac{\partial x^r_{\text{zmp}}}{\partial {\theta^r}} \approx 0\\
        \frac{\partial y^r_{\text{zmp}}}{\partial {\phi^r}} &\approx 0,~~~~~~~~~~~\frac{\partial y^r_{\text{zmp}}}{\partial {\theta^r}} \leq -\frac{M_z}{M}<0.
    \end{aligned}
    \label{zmp_mono}
\end{equation}
\end{proof}

With Proposition \ref{prop2} introduced and graphically illustrated in Figure \ref{move_traj}, we claim that as long as the rectangle derived above is contained within the convex support polygon, the quasi-static mobile motion of the robot is guaranteed to be stable.

\subsection{Traction Optimization}
The results obtained thus far focused on the rollover stability constraint for the robot. However, relocating on 3-D terrain, which may be steep and loose also requires us to consider traction limits in order to ensure the robot does not experience traction loss and unstable sliding. Since mobile manipulators with only two actuated wheels or tracks are more prone to traction loss and resulting slippage, this section will be focused on this type of robot with common symmetrical drive layout, including both wheeled and tracked robots. 

To conduct a force analysis for the robot's two wheels/tracks, we denote the total normal force experienced by the robot as $\boldsymbol{f}_n$, the static friction coefficient as $\mu$, and the friction force generated by the robot's wheels/tracks to remain static on a hill as $\boldsymbol{f}_{f}^0$, where $\|\boldsymbol{f}_{f}^0\|_2\leq \|\mu \boldsymbol{f}_n\|_2$. In order to avoid unplanned sliding as the robot accelerates, the minimum available friction force $f_f^{avail}$ along the heading direction of the mobile base should be optimally distributed between the two wheels/tracks. 

To avoid sliding, the following inequality has to hold: 
\begin{equation}
\|\boldsymbol{f}_f^0+\boldsymbol{f}_f^a\|_2\leq \|\mu\boldsymbol{f}_n\|_2,
\label{acc_bound_lax}
\end{equation}
where $\boldsymbol{f}_f^a$ is the friction force additional to $\boldsymbol{f}_f^0$ that results in base acceleration. Applying a more restrictive bound, we get
\begin{equation}
\begin{aligned}
\|\boldsymbol{f}_f^0\|_2+\|\boldsymbol{f}_f^a\|_2 &\leq \|\mu \boldsymbol{f}_n\|_2\\
\|\boldsymbol{f}_f^a\|_2 &\leq \|\mu \boldsymbol{f}_n\|_2-\|\boldsymbol{f}_f^0\|_2\\
f_f^{avail}&=\|\mu \boldsymbol{f}_n\|_2-\|\boldsymbol{f}_f^0\|_2.
\end{aligned}
\label{acc_bound}
\end{equation}
Note that (\ref{acc_bound}) gives the magnitude of traction force $\boldsymbol{f}_f^a$ a more restrictive bound than (\ref{acc_bound_lax}) and hence it is named ``minimum available traction force". It is a conservative quantity that is ideal for safety assurance. 

Let the distance between the left wheel/track and the ZMP be denoted with $l$, and the distance between the right wheel/track and the ZMP be denoted with $r$. We also denote the magnitude of $\boldsymbol{f}_{n}$ and $\boldsymbol{f}_{f}^0$ with scalars $f_n$ and $f_f^0$, respectively. Then, the normal force distribution on each of the left and right wheel/track can be respectively expressed as:
\begin{equation}
\begin{aligned}
	\|\mu\boldsymbol{f}_{n}\|_2^L = \mu f_n\frac{r}{l+r},~~~
	\|\mu\boldsymbol{f}_{n}\|_2^R = \mu f_n\frac{l}{l+r}.
\end{aligned}
\label{fric_max}
\end{equation}
The friction forces on the two sides that counter sliding due to gravity can be written as 
\begin{equation}
\begin{aligned}
	\|\boldsymbol{f}_f^0\|_2^L = f_f^0\frac{r}{l+r},~~~
	\|\boldsymbol{f}_f^0\|_2^R = f_f^0\frac{l}{l+r}.
\end{aligned}
\label{fric_act}
\end{equation}
The amount of minimum available traction force on each side can then be found using (\ref{acc_bound}) by subtracting (\ref{fric_act}) form (\ref{fric_max}):
\begin{equation}
\begin{aligned}
	f_{f}^{avail,L} \!=\! (\mu f_n-f_f^0)\frac{r}{l+r},~
	f_{f}^{avail,R} \!=\! (\mu f_n-f_f^0)\frac{l}{l+r}.
\end{aligned}
\label{fric_avail}
\end{equation}

Then, traction optimization can be achieved by maximizing the minimum available traction forces on both the left and right wheels/tracks using following formulation:
\begin{equation}
\max_{p_{zmp}} \min(f_{a}^{avail,L}(\boldsymbol{p}_{zmp}),f_{a}^{avail,R}(\boldsymbol{p}_{zmp})).
\label{trac_opt}
\end{equation}
In order to achieve (\ref{trac_opt}), $l=r$ allows available friction forces in (\ref{fric_avail}) to be equal on both wheels. 

It is therefore claimed that vehicle traction is optimized when ZMP lies along the longitudinal center line of $Conv(S)$. However, constraining the ZMP to the center line eliminates the flexibility of robot reconfiguration and would drastically reduce the robot's ability to traverse through terrain. A varying factor that allows the ZMP to deviate from the support polygon's center line on milder terrain will be implemented in the planning phase. 

A comparison between a robot's path generated with and without traction optimization is shown in Figure \ref{fric_compare}. With traction optimization, the robot's path in Figure \ref{fric_opt} is smoother when compared to that in Figure \ref{fric_nopt}. This means that with traction optimization, the robot will tend to descend/ascend the steep sections along the direction of steepest descent/ascent in order to evenly distribute available traction forces between the two driving wheels/tracks to avoid sliding.

\begin{figure}[h!]
\centering
\captionsetup[subfigure]{width=0.9\textwidth,justification=raggedright}
\begin{subfigure}{.25\textwidth}
\centering
\includegraphics[angle=0,origin=c,trim = 45mm 65mm 40mm 95mm, clip, width=4cm]{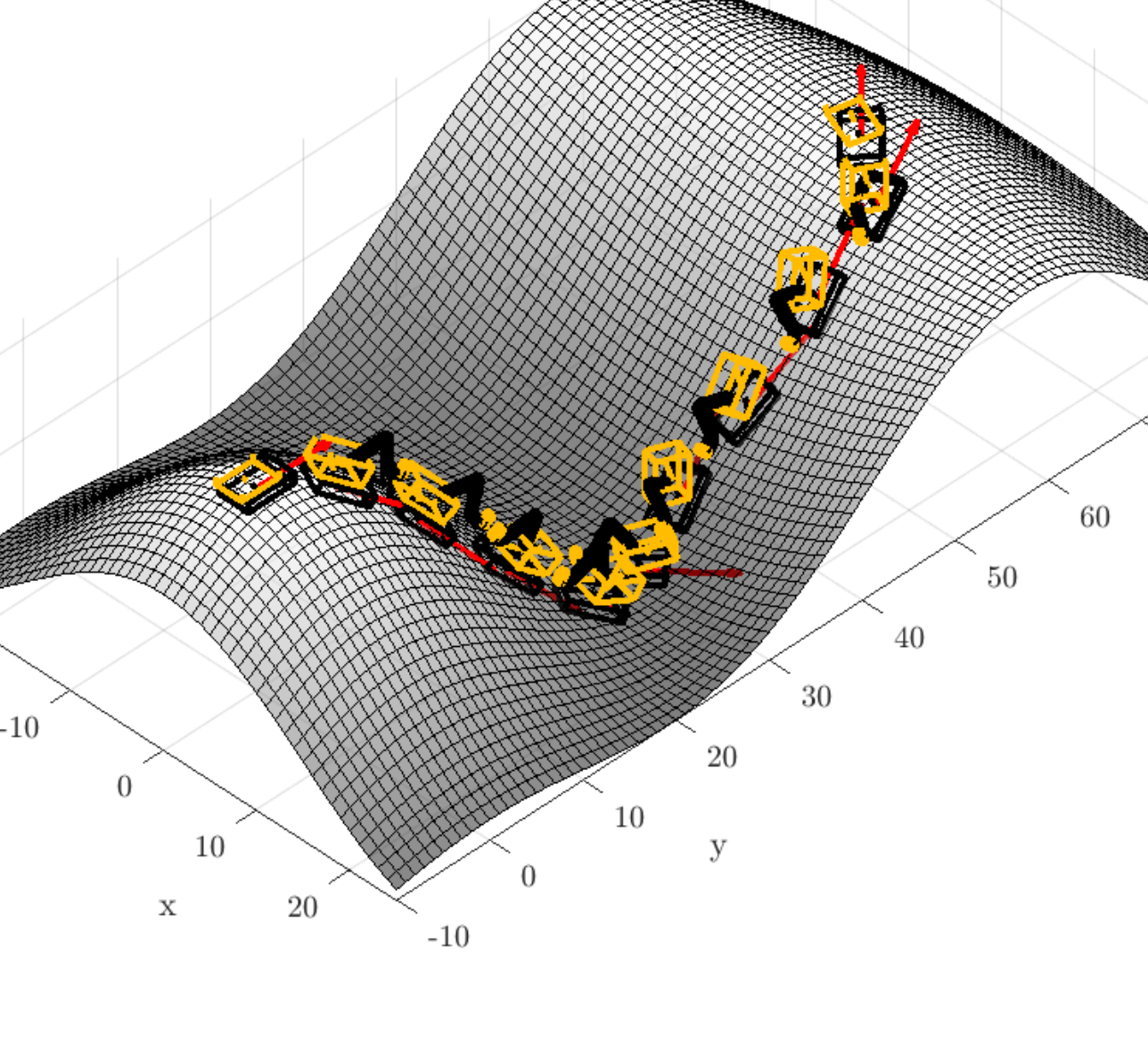}
\caption{Path generated with optimized traction.}
\label{fric_opt}
\end{subfigure}%
\begin{subfigure}{.25\textwidth}
\centering
\includegraphics[angle=0,origin=c,trim = 45mm 65mm 40mm 95mm, clip, width=4cm]{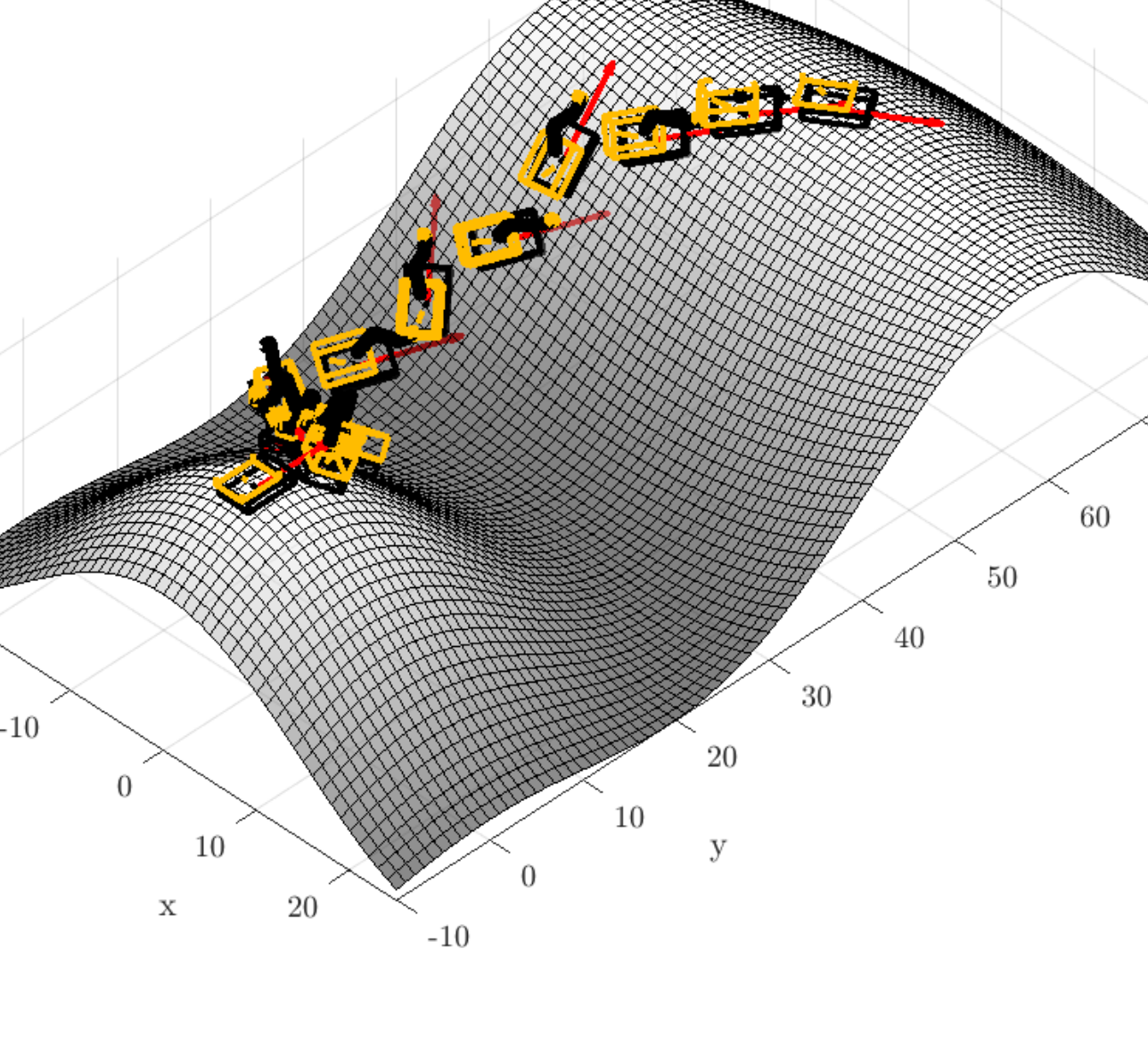}
\caption{Path generated without optimized traction.}
\label{fric_nopt}
\end{subfigure}%
\caption{Two sets of path generated to compare the effect of traction optimization.}
\label{fric_compare}
\end{figure}
\label{trac_opti}

\section{Simulation Results}
In order to showcase the importance of ZMP constraints in mobile manipulator trajectory planning, and the effectiveness of our formulations, the trajectory planning framework is applied to a feller buncher machine. In timber harvesting scenarios, the purpose of a feller buncher is to cut trees, grip them, drop them off, and relocate to another spot to continue performing these tasks. The machine is operated throughout the year, in greatly varying terrain conditions and often on slopes of different grades. In this section, the local manipulation capability of the framework (per Section \ref{manipulation}) will be demonstrated by an example where the machine tries to relocate a heavy tree on a slope. Then, the reconfigurable relocation capability (per Section \ref{relocation}) will be demonstrated by commanding the machine to relocate on a test terrain without carrying a tree.

\subsection{Kinematics Model of a Feller Buncher}
The feller buncher machine is modeled after the Tigercat 855E, and can be described with the schematic diagram in Figure \ref{schematic1}. The machine is considered to be holding onto a tree in the manipulation planning example, and the machine and the tree's geometric and mass parameters are summarized in Table \ref{params}. The machine's support polygon is a $3.23~\text{m
}\!\!\!\times5~\text{m}$ rectangle.

\begin{table}[h!]
\begin{center}
\begin{tabular}{|c|c|c|c|c|} 
\hline
&$a_i$&$\alpha_i$&$d_i$&$\theta_i$\\ [0.5ex] 
\hline
\text{1} &$0$&$0$&$d_1$&$q_1$\\ 
\hline
\text{2} &$0$&$q_2$&$0$&$0$\\
\hline
\text{3} &$0$&$q_3$&$d_3$&$0$\\
\hline
\text{4} &$0$&$q_4$&$d_4$&$0$\\
\hline
\text{5} &$0$&$0$&$d_5$&$q_5$\\
\hline
\end{tabular}
\end{center}
\caption{Denavit-Hartenberg parameters table of the feller buncher.}
\label{dhtable}
\end{table}

The optimal control results in this section are achieved by running the nonlinear optimal control solver GPOPS \cite{patt2014} under MATLAB environment on a Windows desktop with Intel Core i7-4770 3.40 GHz processor.

\subsection{Mapping from Simplified Model to Full Model}
Since the simplified manipulator model has $2$ arm DoFs while the full model has $5$ arm DoFs, to ensure that the ZMP location of the full model follows that of the simplified model, a relationship between the DoFs of the simplified model and those of the full model of the machine needs to be established. To satisfy Assumption 1 b), for the specific machine geometry in this example, we choose to constrain the machine's joint angles such that 
\begin{equation}
[q_3, q_4,q_5]^T = [-\pi-2q_2, \dfrac{\pi}{2}+q_2, \text{constant}]^T,
\label{angle1}
\end{equation}
and the corresponding relationship between joint accelerations follows:
\begin{equation}
[\ddot{q}_3,\ddot{q}_4,\ddot{q}_5]^T = [-2\ddot{q}_2, \ddot{q}_2,0]^T.
\label{angle2}
\end{equation}

Defining the lengths of link 2 and link 4 as $l_2$ and $l_4$, respectively, the relationship between $d$ and $q_2$ can be written as:
\begin{equation}
q_2=\sin^{-1}(\frac{-d}{2l_2}).
\label{d}
\end{equation}
Then, taking the first and second derivative of (\ref{d}) with respect to time, we can obtain:
\begin{equation}
\begin{aligned}
\dot{q}_2 &\!=\! \dfrac{-\dot{d}}{2l_2\sqrt{1-\frac{d^2}{4l_2^2}}}\\
\ddot{q}_2 & = -\dfrac{\ddot{d}}{2l_2\sqrt{1-\frac{d^2}{4l_2^2}}}-\dfrac{d\dot{d}^2}{8l_2^3\sqrt{1-\frac{d^2}{4l_2^2}}^3}.
\end{aligned}
\label{map2}
\end{equation}
Therefore, using the mapping provided in (\ref{angle1}) and (\ref{angle2}), and the relation between joint angles and accelerations as given in (\ref{d}) and (\ref{map2}), the ZMP location of the full kinematics model can be written as $
\boldsymbol{p}_{\text{zmp}}(\tilde{\boldsymbol{x}},\tilde{\boldsymbol{u}})$.

\begin{table*}[h!]
\begin{center}
\begin{tabular}{|c|c|c|c|c|c|c|c|} 
\hline
$\text{Link Number}$&$0$&$1$&$2$&$3$&$4$&$5$&$\text{Tree}$\\ [0.5ex]
\hline
$\text{Length, } l_i \text{ (m)}$&$1.60$&$0.96$&$3.27$&$3.27$&$0.458$&$0.677$&$8$\\
\hline
$\text{Mass, } m_i \text{ (kg)}$&$13000$&$5000$&$2000$&$1000$&$50$&$2600$&$4000$\\
\hline
\end{tabular}
\end{center}
\caption{Mobile manipulator parameters for the test case.}
\label{params}
\end{table*}

\subsection{Full Kinematics vs. Simplified Formulation vs. Phase Plane Method}
In the first test case, we consider the feller buncher to be situated statically on a slope that results in a 30-degree vehicle roll to the left. The machine starts from this initial condition and the goal is to have the cabin yaw 180 degrees towards left. The initial joint angles and velocities are:
\begin{equation}
\begin{aligned}
\boldsymbol{q}(t_0) &= [0, -\pi/6, -2\pi/3, \pi/6, -\pi/2]^T \text{ rad}\\
\dot{\boldsymbol{q}}(t_0) &= [0,0,0,0,0]^T \text{ rad/s},
\end{aligned}
\label{init}
\end{equation}
and the desired final joint angles and velocities are:
\begin{equation}
\begin{aligned}
\boldsymbol{q}(t_f) &= [\pi, -\pi/6, -2\pi/3, \pi/6, -\pi/2]^T \text{ rad}\\
\dot{\boldsymbol{q}}(t_f) &= [0,0,0,0,0]^T \text{ rad/s}.
\end{aligned}
\label{final}
\end{equation}
The joint rate and acceleration constraints are:
\begin{align*}
-\frac{\pi}{4} &\leq \dot{q}_i \leq \frac{\pi}{4} \text{ rad/s}~~~~~\forall~ i\\
-\frac{\pi}{2} &\leq \ddot{q}_i \leq \frac{\pi}{2} \text{ rad/s}^2~~~~\forall~ i.
\end{align*}

\begin{figure}[h!]
\centering
\includegraphics[angle=0,origin=c,trim = 44mm 85mm 45mm 91mm, clip, width=6cm]{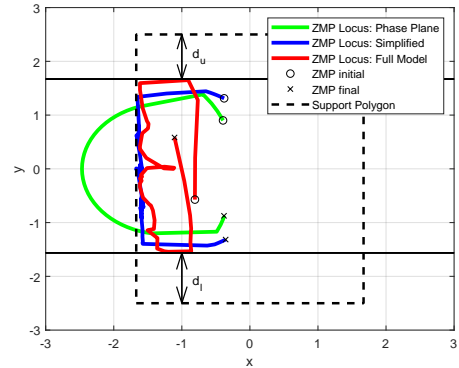}
\caption{ZMP loci of constrained trajectory planning vs. phase plane method.}
\label{planvbang}
\end{figure}

\begin{figure*}[h!]
\centering
\begin{subfigure}{.3\textwidth}
\centering
\includegraphics[angle=0,origin=c,trim = 40mm 70mm 45mm 73mm, clip, width=5cm]{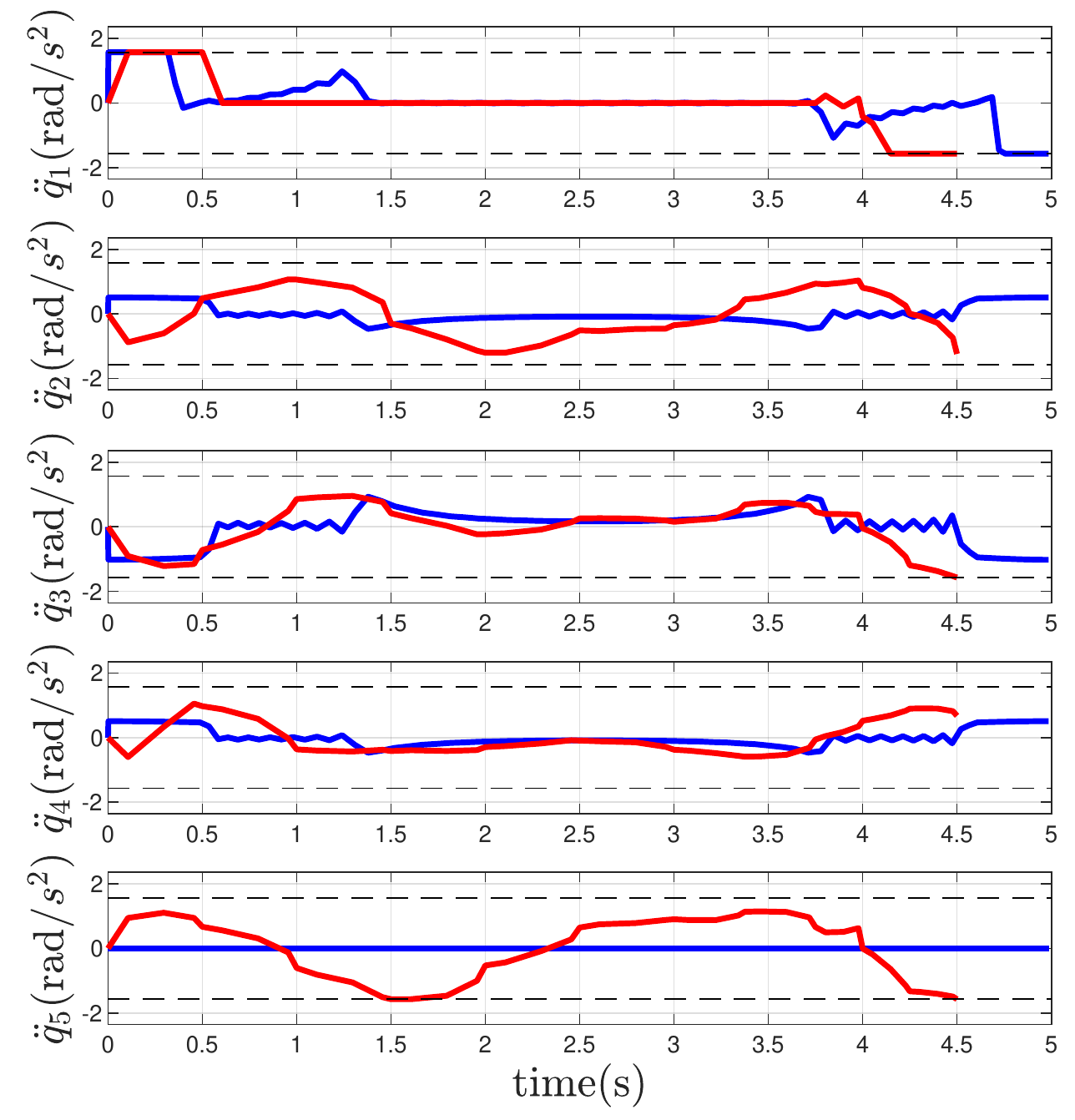}
\caption{Joint accelerations.}
\label{joint_acc}
\end{subfigure}%
\begin{subfigure}{.3\textwidth}
\centering
\includegraphics[angle=0,origin=c,trim = 40mm 70mm 45mm 73mm, clip, width=5cm]{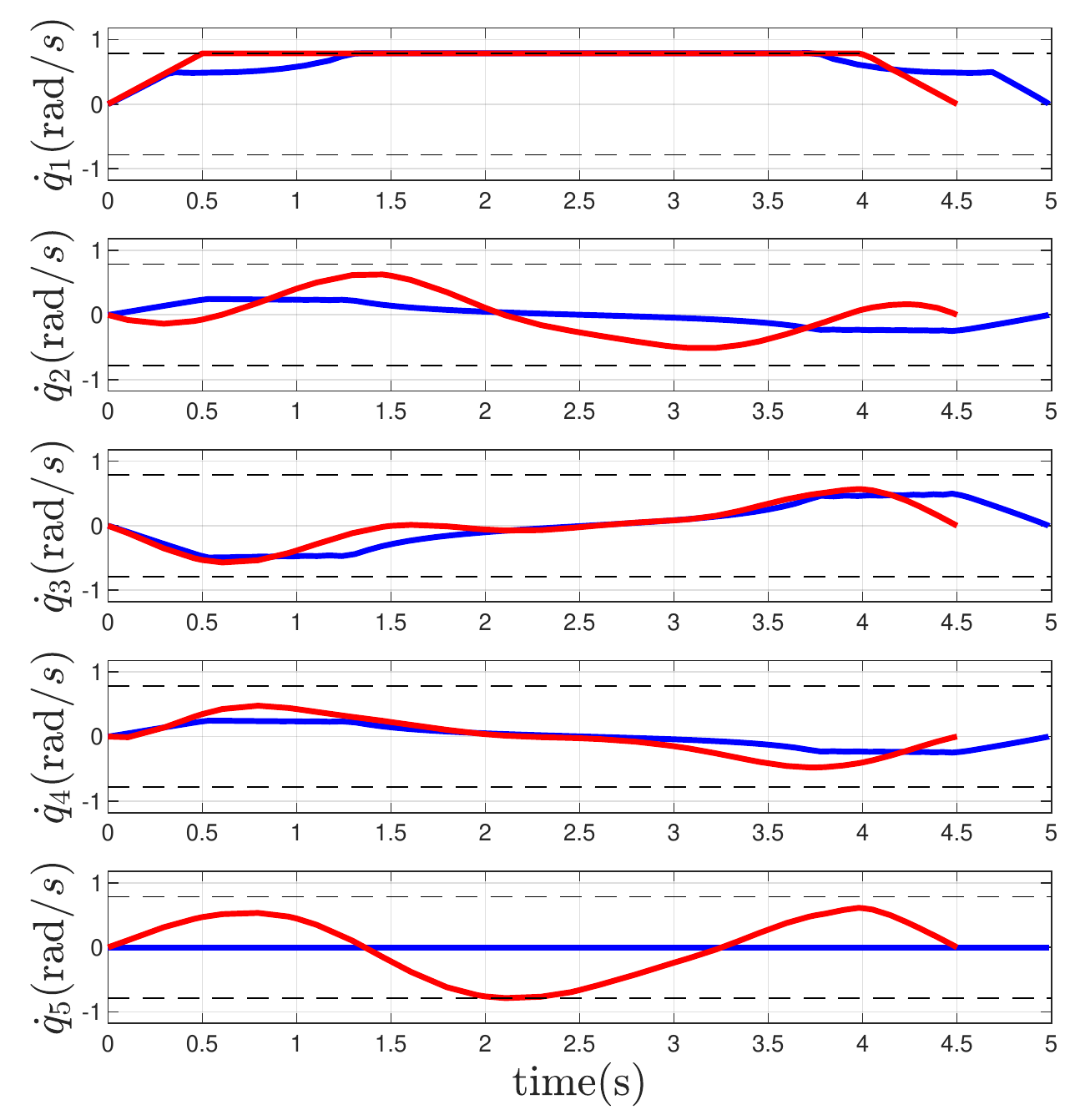}
\caption{Joint velocities.}
\label{joint_vel}
\end{subfigure}%
\begin{subfigure}{.3\textwidth}
\centering
\includegraphics[angle=0,origin=c,trim = 40mm 70mm 45mm 73mm, clip, width=5cm]{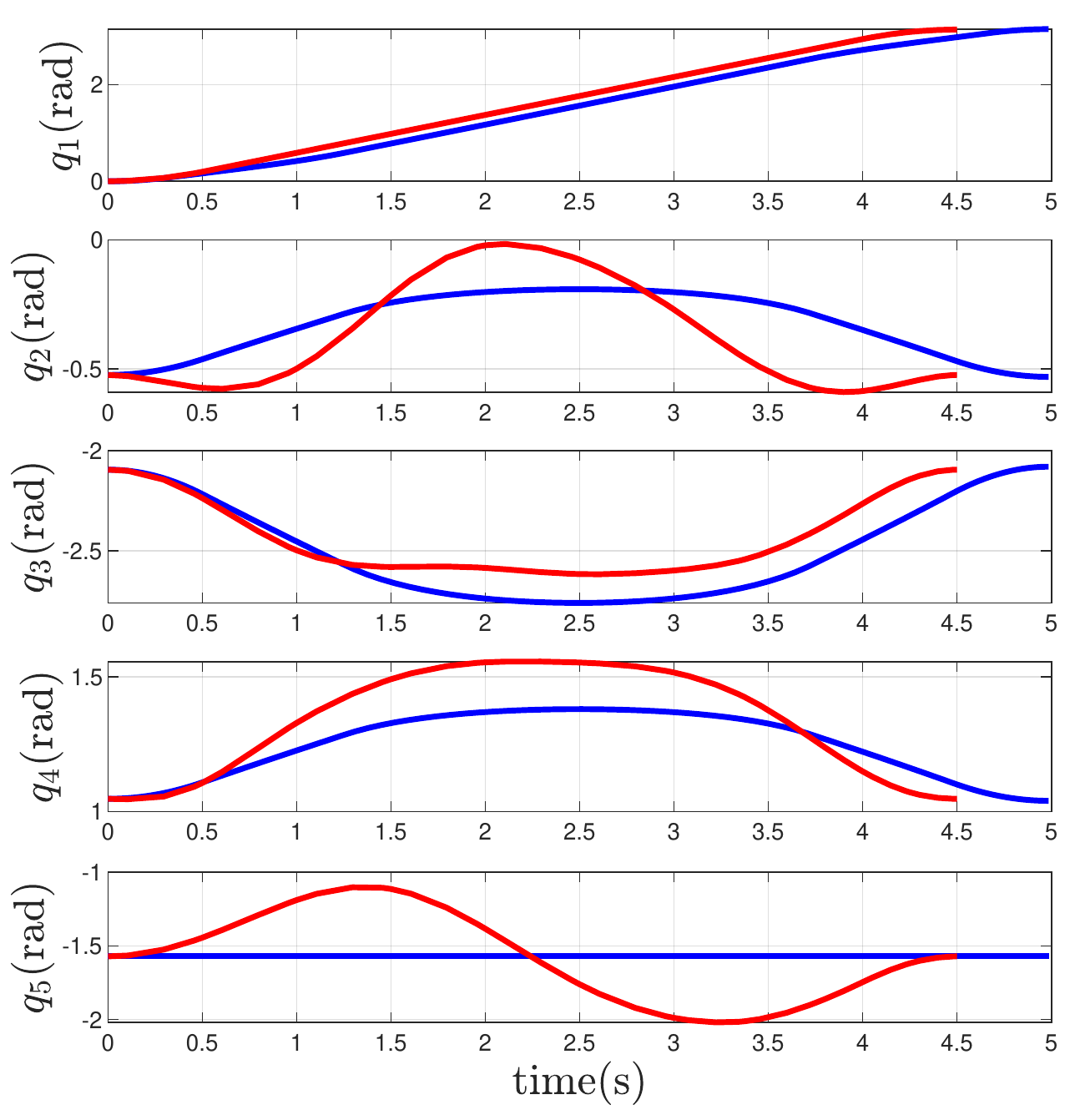}
\caption{Joint angles.}
\label{joint_angle}
\end{subfigure}
\caption{Joint accelerations (a), velocities (b), and angles (c) plots. Blue line indicates results from the simplified model, red line indicates results from the full model, black dashed lines indicate constraints. Note that all blue values of $q_5$ are zero due to this DoF being reduced in the simplified model.}
\label{acc_vel_angle}
\end{figure*}

The ZMP loci of motions generated using the full kinematics optimal trajectory planning formulation (\ref{opt_formulation}), the simplified formulation (\ref{opt_reformulation}), and the classic phase plane method that considers joint angle, rate, and acceleration limits without stability constraints \cite{shin1985} are presented in Figure \ref{planvbang}. These results show that if the feller buncher follows the time optimal trajectory prescribed by the phase plane method, the ZMP locus (green) of the machine will travel outside of the support polygon. In this case, the machine is at risk of rolling over. However, the motion generated through solving (\ref{opt_formulation}) and (\ref{opt_reformulation}) would result in safe (red and blue) ZMP loci. It is noted that although the starting and ending states of the machine for the three trajectory planning methods are the same, the ZMP loci do not start and end at the same points due to different initial accelerations. The unsafe motion generated by the phase plane method takes 4.5 seconds to complete; the safe motion generated by solving the OCP (\ref{opt_formulation}) also takes 4.5 seconds to complete; and the safe motion generated by (\ref{opt_reformulation}) takes 5.0 seconds to complete.

\begin{figure}[h!]
\centering
\includegraphics[angle=0,origin=c,trim = 45mm 86mm 46mm 92mm, clip, width=6cm]{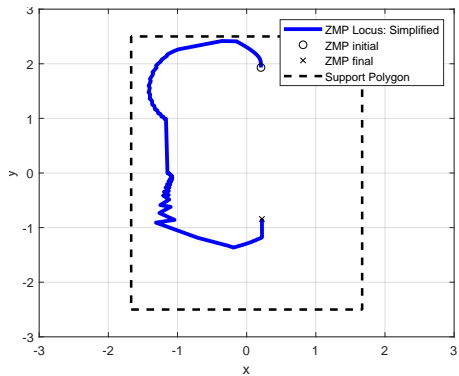}
\caption{ZMP locus throughout mobile manipulation task.}
\label{move}
\end{figure}

For the solutions of (\ref{opt_formulation}) and (\ref{opt_reformulation}), the planned joint accelerations, rates, and angles are displayed in Figures \ref{joint_acc}, \ref{joint_vel}, and \ref{joint_angle}, respectively. It is noted that all initial and final conditions, and state and input constraints are satisfied for both the full and simplified formulations. However, for this example, the computation time to solve (\ref{opt_formulation}) is more than 5 hours, while the computation time to solve (\ref{opt_reformulation}) is 1.34 seconds.

\begin{figure}[h!]
\centering
\includegraphics[angle=0,origin=c,trim = 44mm 90mm 44mm 90mm, clip, width=8cm]{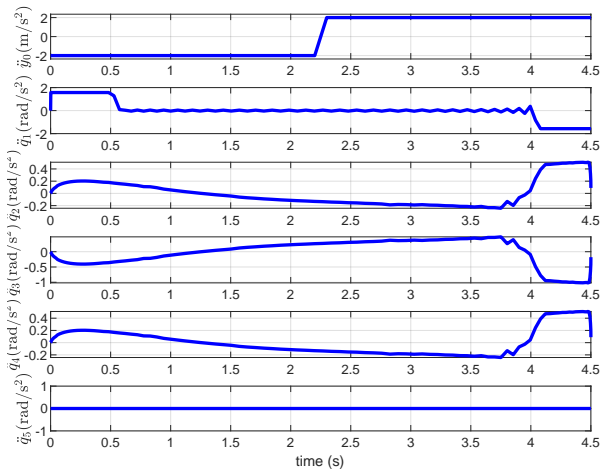}
\caption{Base and joint acceleration of the mobile manipulator.}
\label{move_base}
\end{figure}

\subsection{Simplified Kinematics Model Manipulation-coordinated Base Acceleration Planning}
In the second test case, we consider the feller buncher to be situated on a slope that results in a 15-degree vehicle pitch towards the front and a 15-degree roll to the left. The machine's manipulator arm has to complete the same motion as in the first test case while the mobile base begins at rest, backs up, and goes back to rest during the manipulation. This mobile manipulation task corresponds to a typical occurrence during timber harvesting where the feller buncher has cut a tree on a slope but needs to back up in order to find space to place the tree down. With all joint constraints the same, the mobile base's acceleration has to satisfy $\ddot{y}_0\in [-2,2]~\text{m/s}^2$, and its velocity has to satisfy $\dot{y}_0\in [-10,10]~\text{m/s}$.

It is after obtaining a solution for the optimal trajectory of the manipulator, the safe base acceleration bound is found using equation (\ref{acc_bound}). Figure \ref{move} shows that throughout the task, the ZMP locus never exceeds the edges of the support polygon. Time history of the mobile manipulator's motion is shown in Figure \ref{move_base}. The top plot in Figure \ref{move_base} shows that the mobile base accelerated and decelerated at its maximum magnitude allowed by the corresponding state constraints and the acceleration bounds to achieve time-optimal arrival at the destination. The manipulation motion takes 4.5 seconds to complete and the computation time for arm manipulation is 1.29 seconds, while the time for calculating the base acceleration bound and base trajectory planning is negligible since analytical solutions (\ref{acc_bound}) exist. 

\begin{table}[h!]
\begin{center}
\begin{tabular}{|c|c|} 
\hline
$\text{Base Attitude}$&$\text{Success Rate}$\\ [0.5ex]
\hline
$\text{0-degree Roll}$&$0.989$\\
\hline
$\text{15-degree Roll}$&$0.988$\\
\hline
$\text{20-degree Roll}$&$0.796$\\
\hline
$\text{30-degree Roll}$&$0.406$\\
\hline
\end{tabular}
\end{center}
\caption{Trajectory planning success rate of four base attitudes}
\label{suc_rate}
\end{table}

\subsection{Monte Carlo Simulations for Reconfiguration on Different Slopes}
To test the performance of the simplified model for different initial and final conditions, additional simulations were carried out for four different machine base attitudes of 0,15,20 and 30-degree roll, each with 3000 pairs of randomized initial and final cabin yaw angles from the set $[-2\pi,2\pi]$ with other initial and final joint angles the same as in (\ref{init}) and (\ref{final}). The trajectory planning success rate is presented in Table \ref{suc_rate}. The distribution of successful trajectory planning times of the four base attitudes is presented in Figure \ref{suc_time}. 

It can be observed from these results that the planning success rate decreases with increasing slope angle, and computation time increases with increasing slope angle. The decreasing trend in success rate is due to the more frequent appearance of infeasible initial and final conditions. The distribution in Figure \ref{suc_time} shows that the dimension reduction applied in the simplified model allows the computation time to be low enough for online guidance in most cases.

\begin{figure}[h!]
\centering
\includegraphics[angle=0,origin=c,trim = 45mm 85mm 45mm 85mm, clip, width=8cm]{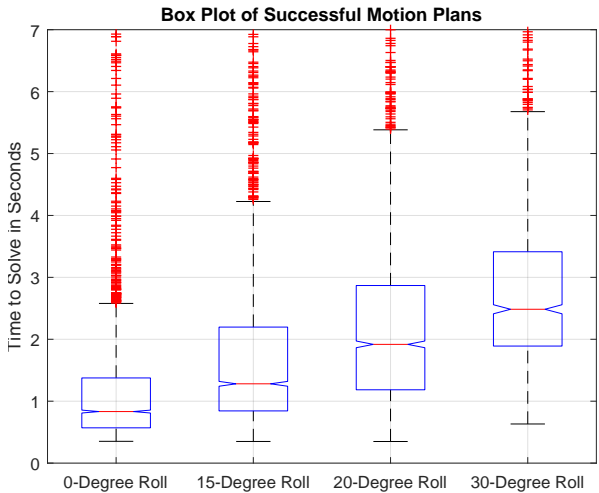}
\caption{Computation time for reconfiguration maneuver of four slope angles using simplified model.}
\label{suc_time}
\end{figure}

\subsection{Mobile Relocation Planning}
To test the performance of mobile relocation trajectory planning, a test terrain described by the sinusoidal function:
\begin{equation}
s_z=10\cos(\frac{\sqrt{{s_x}^2+{s_y}^2}}{10})
\label{surf_fun}
\end{equation}
is created. To formulate the OCP (\ref{opt_formulation}) of this example, the terrain surface function (\ref{surf_fun}) is integrated into the feller buncher's kinematics (\ref{kin_main}).

The machine itself, without the tree, is then commanded to start from the peak of a mountain ($s_x\!\!=\!0$, $s_y\!\!=\!0$, $s_z\!\!=\!10$) and to relocate to a point ($s_x\!\!=\!0$, $s_y\!\!=\!63$, $s_z\!\!\approx \!10$) on the surrounding ridge (see Figure \ref{rrt_path}). Denoting the base's linear velocity as $v$, the initial and final states for the machine's base and arm are:
\begin{equation*}
\begin{aligned}
[x_0,y_0,v,\bar\psi,\dot{\bar{\psi}}]^T(t_0)&=[0,0,0,0,0]^T\\
\boldsymbol{q}(t_0) &= [0, -\pi/6, -2\pi/3, \pi/6, -\pi/2]^T\\
\dot{\boldsymbol{q}}(t_0) &=[0,0,0,0,0]^T\\
[x_0,y_0,v,\bar\psi,\dot{\bar{\psi}}]^T(t_f)&=[0,63,0,\text{free},0]^T\\
\dot{\boldsymbol{q}}(t_f) &=[0,0,0,0,0]^T,
\end{aligned}
\end{equation*}
and we note that the final heading angle of the base and arm joint angles are left free. The constraints on the mobile base's linear and angular rates and accelerations are:
\begin{equation*}
\begin{aligned}
0\leq &v\leq 2.78~~\text{m}/\text{s}\\
-1\leq &u_a\leq 1~~\text{m}/\text{s}^2\\
-2\leq &\dot{\psi}\leq 2~~\text{rad}/\text{s}\\
-\frac{2}{3}\leq &u_{\psi}\leq \frac{2}{3}~~\text{rad}/\text{s}^2.
\end{aligned}
\end{equation*}

Algorithm \ref{TRRT} first generates a quasi-static path with guaranteed stability for continuous motion for the machine indicated by the blue crosses in Figure \ref{rrt_path}. The resulting path requires the machine to reconfigure its cabin angle $q_1$, as shown by Figures \ref{pre_change} and \ref{post_change}, to shift the ZMP so the machine does not rollover towards the front as it is descending the slope. The sampling point where this configuration change happens is highlighted in Figure \ref{rrt_path_sub} by a red circle.

\begin{figure}[h!]
\centering
\captionsetup[subfigure]{width=0.9\textwidth,justification=raggedright}
\begin{subfigure}{.25\textwidth}
\centering
\includegraphics[angle=0,origin=c,trim = 45mm 65mm 40mm 95mm, clip, width=4cm]{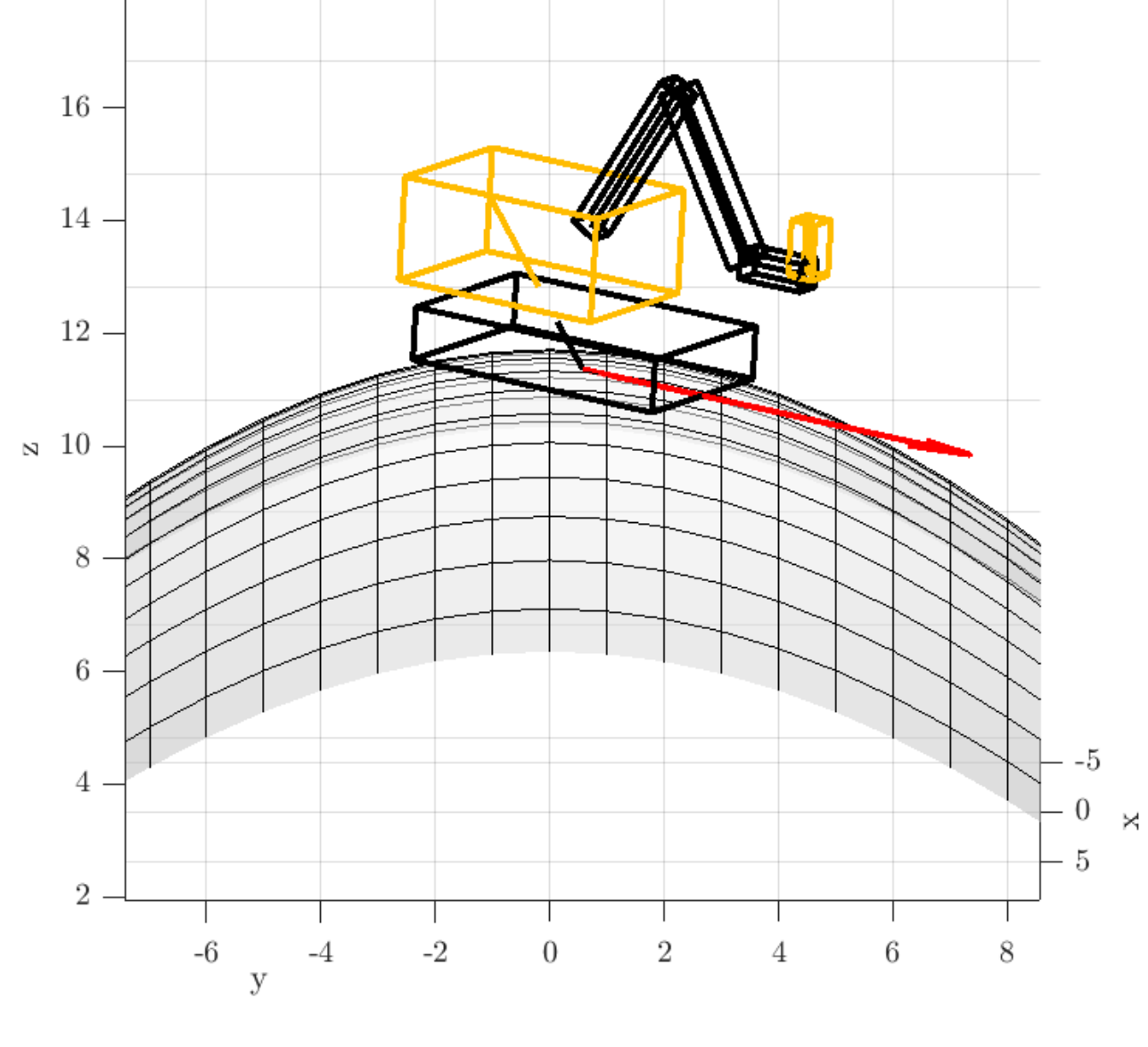}
\caption{Machine configuration before change.}
\label{pre_change}
\end{subfigure}%
\begin{subfigure}{.25\textwidth}
\centering
\includegraphics[angle=0,origin=c,trim = 45mm 65mm 40mm 95mm, clip, width=4cm]{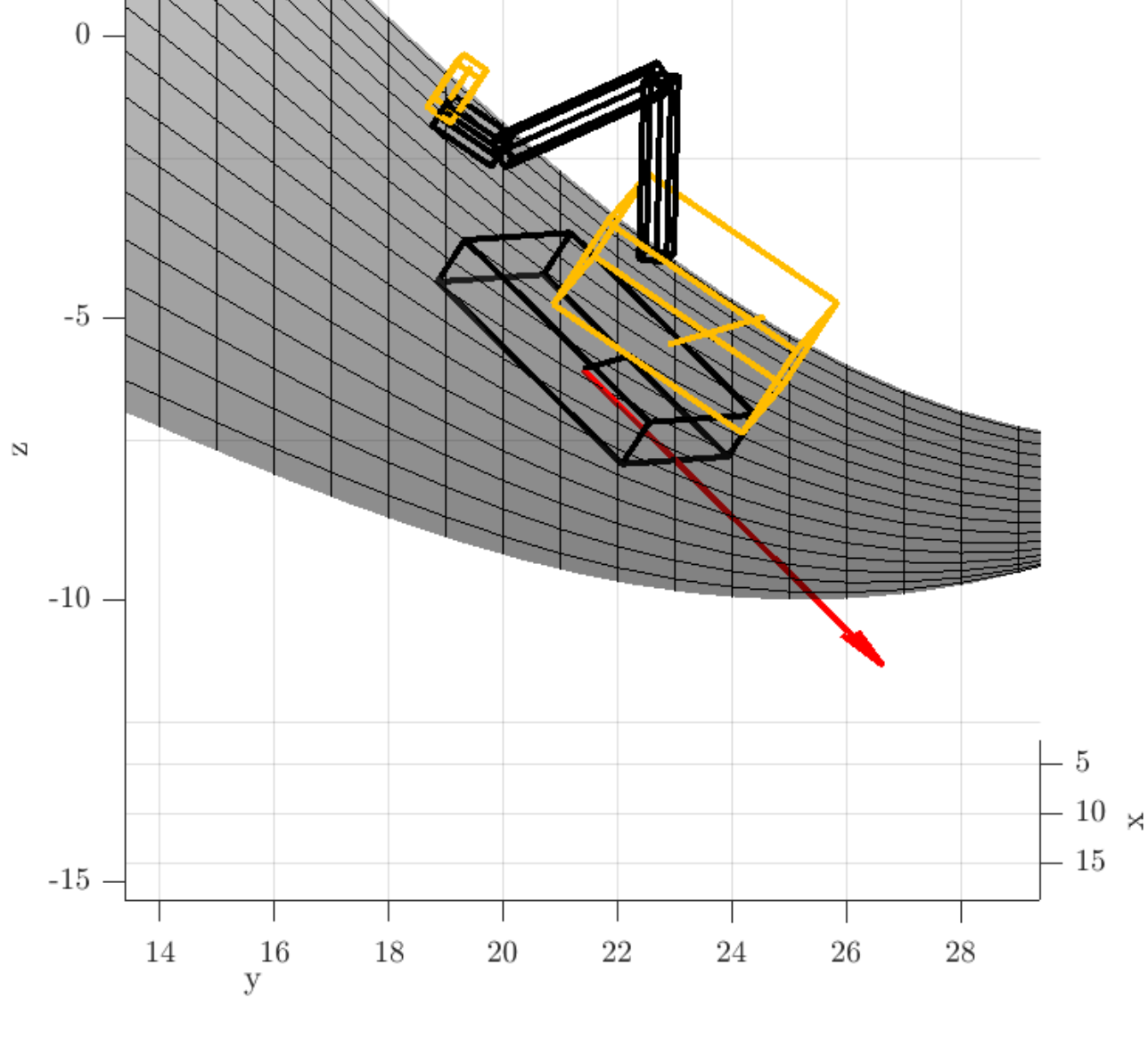}
\caption{Machine configuration after change.}
\label{post_change}
\end{subfigure}%
\vskip\baselineskip
\begin{subfigure}{.5\textwidth}
\centering
\includegraphics[angle=0,origin=c,trim = 25mm 75mm 25mm 75mm, clip, width=8cm]{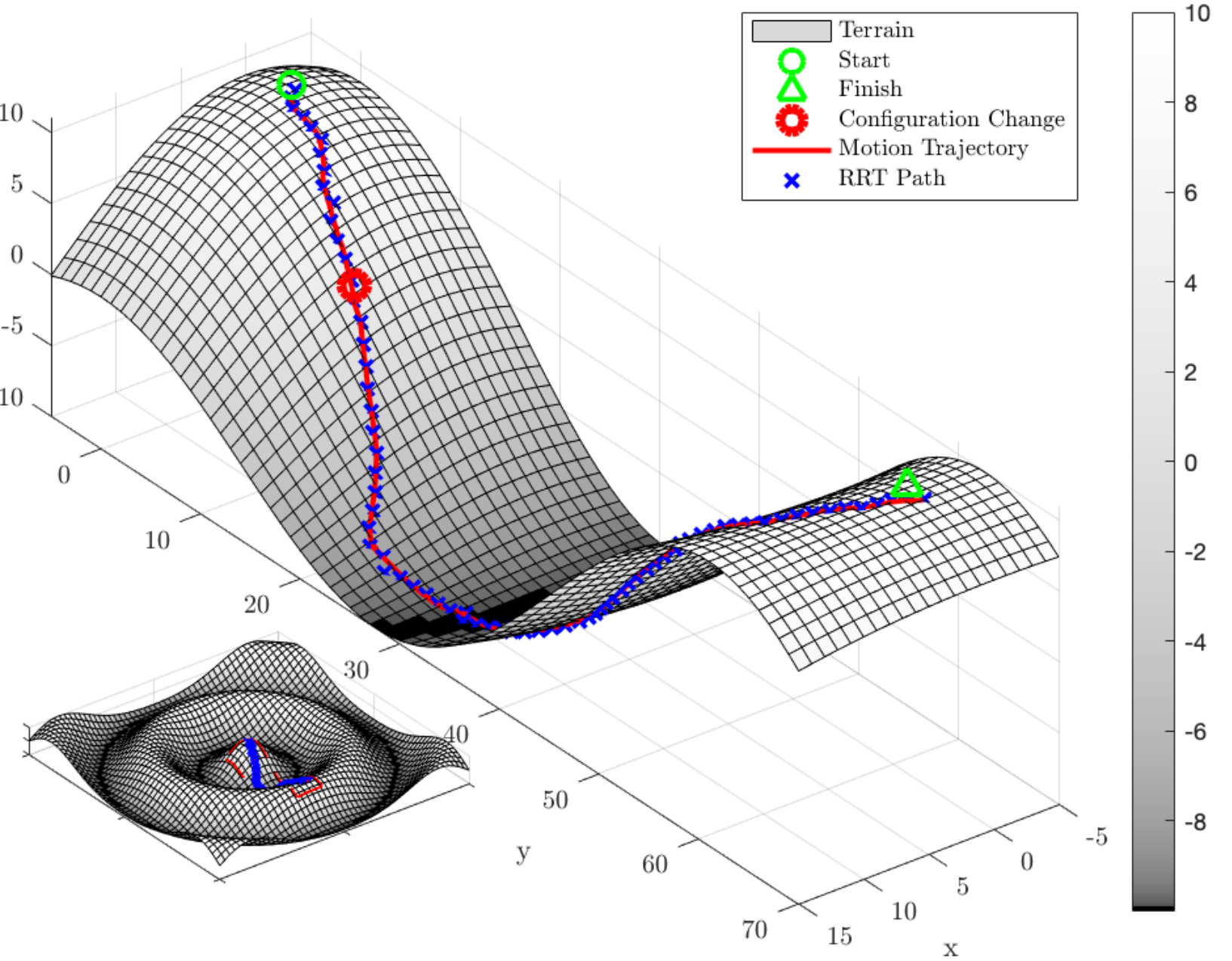}
\caption{Overview of the terrain, quasi-static path generated from the sampling-based planner (blue cross), and trajectory (red line).}
\label{rrt_path_sub}
\end{subfigure}
\caption{Planned path and point of reconfiguration of the feller buncher machine on a sinusoidal test terrain.}
\label{rrt_path}
\end{figure}


Then, based on the quasi-static path and configuration change generated by Algorithm \ref{TRRT}, an initial guess is given to GPOPS\cite{patt2014}, a nonlinear OCP solver, to generate linear and angular accelerations for the machine to allow for time-optimal relocation. The solver generates the trajectory of each segment with the initial and final location of a segment being two consecutive sampling points. To create an initial guess of each segment for the OCP solver, the machine is constrained to be static at each sampling point but is given small constant linear velocity and heading rate when it is between sampling points to allow dynamic stability constraint satisfaction. The resulting trajectories of all segments are then combined. The continuous path of the machine is represented by the red line in Figure \ref{rrt_path}, the motion trajectory is shown in Figure \ref{optimo}, and the ZMP trajectory is shown in Figure \ref{optimo2}. Note that the machine's velocity remains 0 in segment 14 due to arm reconfiguration. 
\begin{figure}[h!]
\centering
\includegraphics[angle=0,origin=c,trim = 5mm 72mm 6mm 75mm, clip, width=8.2cm]{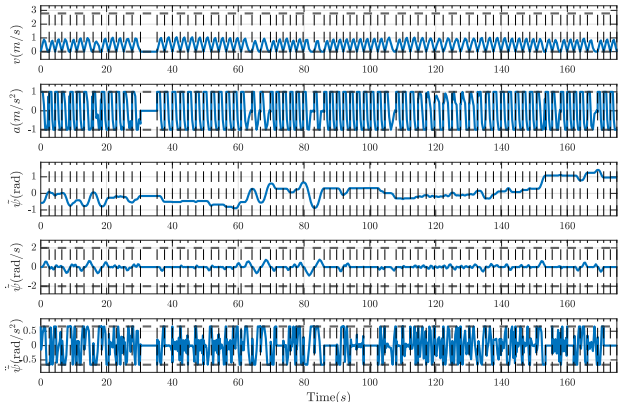}
\caption{Motion trajectory generated by iteratively solving OCP of each segment.}
\label{optimo}
\end{figure}

\begin{figure}[h!]
\centering
\includegraphics[angle=0,origin=c,trim = 10mm 100mm 7mm 100mm, clip, width=8.2cm]{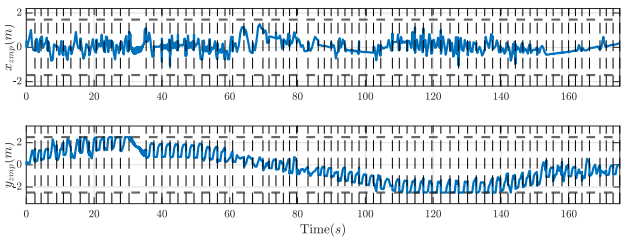}
\caption{ZMP trajectory generated by iteratively solving OCP of each segment.}
\label{optimo2}
\end{figure}

The result in Figure \ref{optimo} shows that all state and stability constraint variables remained within their bounds, indicated by black dashed lines, and the overall trajectory takes $\approx 180$ seconds to complete. Each motion segment separated by vertical dashed lines took $\approx0.3$ seconds to calculate, and the entire trajectory took 18.30 seconds to plan. This shows that the framework proposed in this paper has the potential to be implemented to achieve online dynamically stable mobile manipulation planning. 

The resulting ZMP trajectory is shown in Figure \ref{optimo2} with black dashed horizontal lines indicating the boundaries of $Conv(S)$. During sampling-based path planning, the deviation of $x_{zmp}$ from the centerline of $Conv(S)$ is constrained to change linearly with respect to the slope angle from $50\%$ of the total withth of $Conv(S)$ on flat ground to $10\%$ on $45^o$ slopes. The machine's $x_{zmp}$ is shown to have stayed close to $0$, i.e., the centerline of the support polygon, during ascent and descent as a result of the traction-optimized path as mentioned in Section \ref{trac_opti}. The location of $y_{zmp}$ is shown to generally move forward as the machine is moving down the slope and then move backward as the machine is climbing the slope.

However, due to the quasi-static inter-segment constraint, the overall motion of the machine is not satisfactory: the velocities and accelerations are highly oscillatory (see the subfigures for $a$ and $\ddot{\psi}$ ) and can cause less than ideal performance, along with excessive hardware wear. To remedy this, a receding horizon OCP solution scheme is implemented. 

For the result shown in Figure \ref{optimo}, more than one trajectory segment can be treated as a whole and thus used as the initial guess of a new OCP segment. The length of time interval in the new combined segment will be called ``horizon length". For this specific example, the horizon length is chosen to be $\approx 4$ seconds. When GPOPS cannot converge to a solution within a user specified time threshold, the solver is terminated and result computed for the original segment is kept for the machine's execution. This way, the motion trajectory can be optimized where possible without forcing the machine to stop and wait for a result. 

The receding horizon scheme enhanced motion trajectory is shown in Figure \ref{optimo_en}, with the corresponding ZMP trajectory shown in Figure \ref{optimo_en2}. It can be observed that, compared to Figure \ref{optimo}, the motion trajectory has become significantly smoother. The relocation completion time has also been shortened to $\approx 105$ seconds due to the enhancement. Both the state and dynamic stability constraints are shown to be respected. The joint motion trajectories of the arm are shown in Figure \ref{optijo}. The joint variables vary only in one segment which corresponds to 0 velocity of the machine's base, when the robot reconfigures itself. The successful receding horizon re-planning segments took a total 8.78 seconds to complete, while the unsuccessful segments took a total of 14 seconds with the computational time threshold set of 0.5 seconds.


\begin{figure}[h!]
\centering
\includegraphics[angle=0,origin=c,trim = 5mm 72mm 6mm 75mm, clip, width=8.2cm]{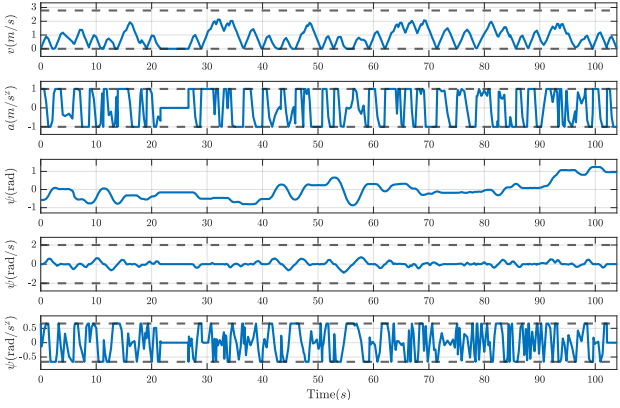}
\caption{Motion trajectory generated through iteratively solving OCP with existing motion trajectory and a horizon length of 20 time steps.}
\label{optimo_en}
\end{figure}

\begin{figure}[h!]
\centering
\includegraphics[angle=0,origin=c,trim = 10mm 100mm 7mm 100mm, clip, width=8.2cm]{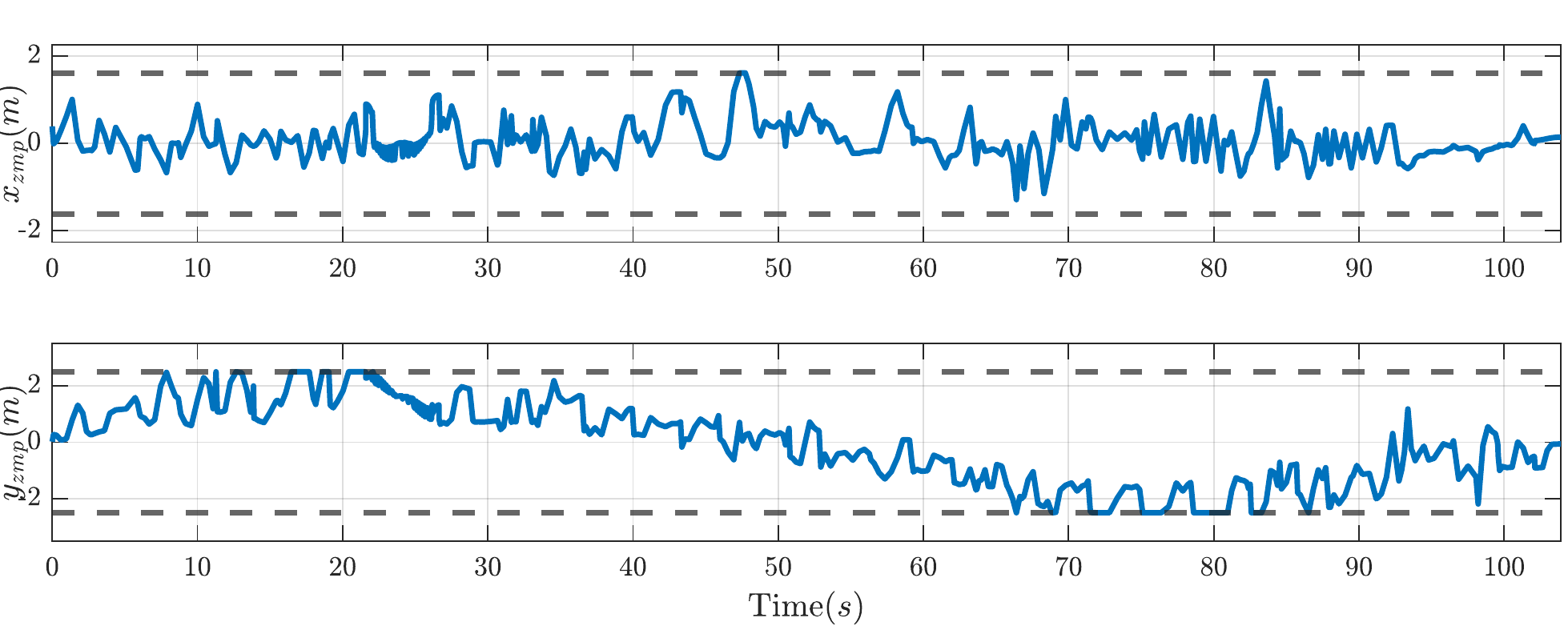}
\caption{ZMP trajectory generated through iteratively solving OCP with existing motion trajectory and a horizon length of 20 time steps.}
\label{optimo_en2}
\end{figure}

\begin{figure}[h!]
\centering
\includegraphics[angle=0,origin=c,trim = 12mm 74mm 4mm 74mm, clip, width=8.5cm]{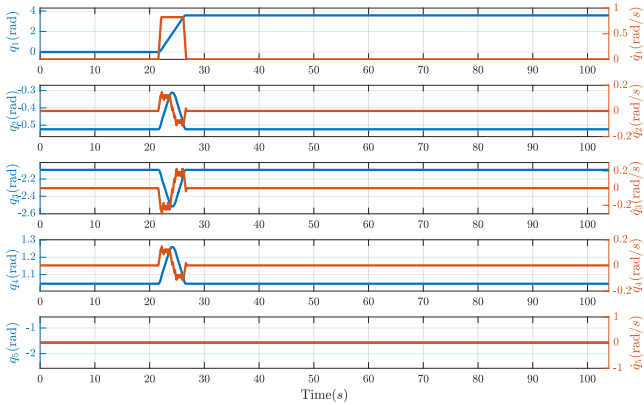}
\caption{Joint motion trajectories of the machine with joint angles shown in red and joint velocities shown in blue.}
\label{optijo}
\end{figure}

%

\section{Conclusions and Future Work}
To conclude, this paper presented a framework that allows online optimal mobile manipulation planning to allow for online guidance for a robot to work safely and efficiently on rough terrain. This is achieved through the formulation of a trajectory planning problem as an optimal control problem with an additional dynamic stability constraint. To reduce the problem dimension, the mobile manipulation task is first divided into two separate stages: the manipulation stage and the relocation stage. The online trajectory planning of the manipulation stage is achieved through dimension reduction with the help of the known inverse kinematics mapping of the robot. Then, to solve the relocation problem, a sampling-based path planning algorithm that optimizes slope traction availability and generates results guaranteed to meet nonlinear dynamic stability constraints when continuously executed is implemented. The robot path and reconfiguration command generated by the sampling-based algorithm are then taken to warm start a nonlinear OCP solver. The resulting trajectory is further optimized and smoothed with the receding horizon OCP re-planning scheme.

The framework presented in this paper is the first attempt at a task-based online optimal planning algorithm for a mobile manipulator that accommodates the nonlinear dynamic stability constraint. Some improvements can still be made to the framework to be implemented in a wider variety of scenarios. For example, the terrain map used in this paper only assumes a rough knowledge of the terrain features. However, as the robot traverses through the terrain, more detailed terrain features can be captured by its on-board sensors. With a more detailed local terrain map, a re-planning method or a stability-guaranteed control law can be developed to help the robot deal with local terrain complexities. 

\begin{appendix}

\begin{proof}[Proof of Theorem \ref{t1}]
During the execution of a quasi-statically planned path that is made of straight line segments, the mobile robot's motion can be decomposed into two motion primitives: forward and turning. Differentiating static variables and dynamic variables at the same location using superscript $\cdot^{\text{static}}$ and $\cdot^{\text{dynamic}}$, the distance between $\boldsymbol{p}_{zmp}^{\text{dynamic}}$ and $\boldsymbol{p}_{zmp}^{\text{static}}$ can be characterized using the 2-norm: $||\boldsymbol{p}_{zmp}^{\text{dynamic}}-\boldsymbol{p}_{zmp}^{\text{static}}||_2$. Here $\boldsymbol{p}_{zmp}^{\text{dynamic}}$ and $\boldsymbol{p}_{zmp}^{\text{static}}$ can both be found by solving (\ref{zmp_origin}).

Under the assumption that straight line segments are small so that the robot's attitude change during forward motion is negligible,we have $\ddot{x}_i^{\text{dynamic}}=\ddot{x}_i^{\text{static}}=0$, and $\ddot{z}_i^{\text{dynamic}}=\ddot{z}_i^{\text{static}}=0$ $\forall i\in\{0,\cdots,n\}$. Therefore, $x_{zmp}^{\text{dynamic}}=x_{zmp}^{\text{static}}$. As a result, we get the following:
\begin{equation}
\begin{aligned}
\left|\left|\boldsymbol{p}_{zmp}^{\text{dynamic}}-\boldsymbol{p}_{zmp}^{\text{static}}\right|\right|_2 &=\left|y_{zmp}^{\text{dynamic}}-y_{zmp}^{\text{static}}\right|\\
&=\left|\frac{\sum_im_i \ddot{y}_i^{\text{dynamic}}z_i}{\sum_im_ig_z}\right|\\
&\leq \left|\frac{\sum_im_i z_i}{\sum_im_ig_z}\right|\left|\ddot{y}^{\text{dynamic}}\right|\\
&=\left|\frac{\sum_im_i z_i}{\sum_im_ig_z}\right|\left|u_a\right|
\end{aligned}
\label{forward_bound}
\end{equation}

Derivation in (\ref{forward_bound}) shows that, the distance between $\boldsymbol{p}_{zmp}^{\text{static}}$ and $\boldsymbol{p}_{zmp}^{\text{dynamic}}$ during forward motion is proportionally bounded by the magnitude of input $u_a$.

For turning motion, every link on the robot goes through a rotation along the $z_0$-axis that can be described by the following rotation matrix:
\begin{equation}
R_{yaw}=\begin{bmatrix}
\cos{\psi}&-\sin{\psi}&0\\
\sin{\psi}&\cos{\psi}&0\\
0&0&1
\end{bmatrix}.
\end{equation}
With the mobile base fixed during the turn, the linear acceleration of each link's center of mass can be written as:
\begin{equation}
\!\begin{bmatrix}
\ddot{x}_i\\
\ddot{y}_i\\
\ddot{z}_i
\end{bmatrix}
\!\!\!=\underbrace{\!\!\!
\begin{bmatrix}
-\cos{\psi}\dot{\psi}^2-\sin{\psi}u_{\psi}&\sin{\psi}\dot{\psi}^2-\cos{\psi}u_{\psi}&0\\
-\sin{\psi}\dot{\psi}^2+\cos{\psi}u_{\psi}&-\cos{\psi}\dot{\psi}^2-\sin{\psi}u_{\psi}&0\\
0&0&0
\end{bmatrix}\!\!\!}_{\ddot{R}_{yaw}}
\begin{bmatrix}
x_i\\
y_i\\
z_i
\end{bmatrix}\!\!\!.
\end{equation}
Then, the distance can be written as:
\begin{equation}
\begin{aligned}
&\left|\left|\boldsymbol{p}_{zmp}^{\text{dynamic}}\!\!-\!\!\boldsymbol{p}_{zmp}^{\text{static}}\right|\right|_2\\
&=
\left|\left|\left[x_{zmp}^{\text{dynamic}}-x_{zmp}^{\text{static}},y_{zmp}^{\text{dynamic}}-y_{zmp}^{\text{static}}\right]^T\right|\right|_2\\
&=\left|\left|\left[\frac{\sum_i m_i\ddot{x}_i^{\text{dynamic}}\!z_i}{\sum_im_ig_z},\frac{\sum_im_i\ddot{y}_i^{\text{dynamic}}z_i}{\sum_im_ig_z} \right]^T\right|\right|_2\\
&\leq \left|\frac{\sum_i m_i\ddot{x}_i^{\text{dynamic}}\!z_i}{Mg_z}\right|+\left|\frac{\sum_i m_i\ddot{y}_i^{\text{dynamic}}\!z_i}{Mg_z}\right|\\
&=\left|\left[\frac{m_0z_0}{Mg_z},\cdots,\frac{m_nz_n}{Mg_z}\right]\left[\ddot{x}_0^{\text{dynamic}},\cdots,\ddot{x}_n^{\text{dynamic}} \right]^T \right|\\
&~~~~+\left|\left[\frac{m_0z_0}{Mg_z},\cdots,\frac{m_nz_n}{Mg_z}\right]\left[\ddot{y}_0^{\text{dynamic}},\cdots,\ddot{y}_n^{\text{dynamic}} \right]^T \right|\\
&\leq \left|\left|\left[\frac{m_0z_0}{Mg_z},\cdots,\frac{m_nz_n}{Mg_z}\right]\right|\right|_2 \bigg(\left|\left|\left[\ddot{x}_0^{\text{dynamic}},\cdots,\ddot{x}_n^{\text{dynamic}} \right]\right|\right|_2\\
&~~~~+\left|\left|\left[\ddot{y}_0^{\text{dynamic}},\cdots,\ddot{y}_n^{\text{dynamic}} \right]\right|\right|_2\bigg)\\
&\leq 2\left|\left|\left[\frac{m_0z_0}{Mg_z},\cdots,\frac{m_nz_n}{Mg_z}\right]\right|\right|_2 \sum_i(|x_i|\dot{\psi}^2+|x_i||u_{\psi}|+\\
&|y_i|\dot{\psi}^2+|y_i||u_{\psi}|).
\end{aligned}
\label{turn_bound}
\end{equation}

From inequality (\ref{forward_bound}) and (\ref{turn_bound}), it can be inferred that there exist trajectories of $u_a$ and $u_{\phi}$ such that, $\forall~t\in[t_0,t_f]$, $v>0$, $\dot{\psi}>0$, and $\left|\left|\boldsymbol{p}_{zmp}^{\text{dynamic}}\!\!-\!\!\boldsymbol{p}_{zmp}^{\text{static}}\right|\right|_2 \leq \sigma$ for arbitrarily small $\sigma>0$.  Hence, there always exists a trajectory of $\boldsymbol{u}$ such that $\dot \tau>0$ over the unit interval $I$.

\end{proof}
\end{appendix}

\section*{Acknowledgement}
This work was supported by the National Sciences and
Engineering Research Council (NSERC) Canadian
Robotics Network (NCRN), the McGill Engineering Doctoral
Awards and Summer Undergraduate Research in Engineering
(SURE) programs. The authors would also like to thank Centre de Formation Professionnelle Mont-Laurier for their help during field work.

\end{CJK}
\end{document}